\documentclass[nohyperref]{article}
\usepackage[a4paper, total={6in, 9in}]{geometry}
\usepackage{microtype}
\usepackage{graphicx}
\usepackage{subfigure}
\usepackage{booktabs} 
\usepackage{hyperref}
\hypersetup{colorlinks=true,
    linkcolor=magenta,      
    urlcolor=cyan,
    citecolor=blue}

\usepackage{amsmath}
\usepackage{amssymb}
\usepackage{mathtools}
\usepackage{amsthm}

\usepackage[capitalize,noabbrev]{cleveref}

\theoremstyle{plain}

\newtheorem{proposition}{Proposition}[section]

\newtheorem{fact}{Fact}[section]

\theoremstyle{definition}
\newtheorem{definition}{Definition}[section]

\theoremstyle{remark}

\usepackage[algo2e,ruled,nofillcomment]{algorithm2e} 

\SetCommentSty{mycommfont}

\crefname{algocfline}{Algorithm}{Algorithms}
\Crefname{algocfline}{Algorithm}{Algorithms}

\usepackage[textsize=tiny]{todonotes}

\usepackage{bbm}
\usepackage{xcolor}
\usepackage{thmtools}
\usepackage{thm-restate}
\usepackage{amssymb}
\usepackage{thmtools}
\usepackage{comment}
\usepackage[mathscr]{eucal}

\DeclareMathOperator{\nt}{\zeta}
\DeclareMathOperator{\s}{\sigma}
\DeclareMathOperator{\D}{\Delta}
\DeclareMathOperator{\mh}{\hat{\mu}}
\DeclareMathOperator{\Ut}{\mathcal{U}_t}
\DeclareMathOperator{\Vt}{\mathcal{A}_t}
\DeclareMathOperator{\Wt}{\mathcal{B}_t}
\DeclareMathOperator{\Yt}{\mathcal{E}_t}

\DeclareMathOperator{\sm}{\bar{\sigma}}
\DeclareMathOperator{\Feedback}{\mathbf{\Sigma}}

\definecolor{ashgrey}{rgb}{0.7, 0.75, 0.71}

\newcommand\event[1]{\mathop{\mathbb{I}\left(#1\right)}}

\newcommand\Ex[2]{\mathop{\underset{#1}{\mathbb{E}}\left[#2\right]}}

\title{Asymptotically-Optimal Gaussian Bandits with Side Observations}

\author{Alexia Atsidakou* \\ \textit{\small{Department of Electrical and Computer Engineering}} \\ \textit{\small{University of Texas at Austin}}  \and Orestis Papadigenopoulos*\\ \textit{\small{Department of Computer Science}}\\ \textit{\small{University of Texas at Austin}} \and
Constantine Caramanis \\ \textit{\small{Department of Electrical and Computer Engineering}} \\ \textit{\small{University of Texas at Austin}}  \and Sujay Sanghavi \\ \textit{\small{Department of Electrical and Computer Engineering}}\\ \textit{\small{University of Texas at Austin}} 
\and Sanjay Shakkottai \\ \textit{\small{Department of Electrical and Computer Engineering}}\\ \textit{\small{University of Texas at Austin}}
}
\date{}
\begin{document}

\maketitle

\def\thefootnote{*}\footnotetext{These authors contributed equally. Corresponding author: Alexia Atsidakou, atsidakou@utexas.edu\\ 
\textit{Proceedings of the $39$th International Conference on Machine Learning, Baltimore, Maryland, USA, PMLR 162, 2022. Copyright 2022
by the author(s).}}

\vskip 0.3in

\begin{abstract}
We study the problem of Gaussian bandits with general side information, as first introduced by Wu, Szepesv\'{a}ri, and Gy\"{o}rgy. In this setting, the play of an arm reveals information about other arms, according to an arbitrary {\em a priori} known {\em side information} matrix: each element of this matrix encodes the fidelity of the information that the ``row'' arm reveals about the ``column'' arm. 
In the case of Gaussian noise, this model subsumes standard bandits, full-feedback, and graph-structured feedback as special cases. 
In this work, we first construct an LP-based asymptotic instance-dependent lower bound on the regret. The LP optimizes the cost (regret) required to reliably estimate the suboptimality gap of each arm. This LP lower bound motivates our main contribution: the first known asymptotically optimal algorithm for this general setting.
\end{abstract}


\section{Introduction}
\label{section:Introduction}
In the stochastic online learning framework, a player sequentially selects from a set of available actions (or ``arms'') and collects a stochastic reward associated with the chosen action. Under the objective of maximizing the (expected) cumulative reward collected over a number of rounds, the ``complexity'' of an instance is greatly impacted by the nature of the feedback the player receives at the end of each round. 
In the {\em full-feedback} case, for instance, where the player observes the realized rewards of all actions, the most reasonable (and provably optimal) strategy is to greedily select at each round the action with maximum estimated mean reward computed using the observed samples.
In the {\em bandit-feedback} case \cite{lai19854,Bubeck2012}, on the other hand, where only the reward of the chosen action is revealed to the player, more sophisticated ideas have to be leveraged in order to align the two conflicting objectives of the problem: {\em exploration} (learning the best action) and {\em exploitation} (not playing suboptimal actions).

In a number of real world applications, however, the feedback does not fall into any of the above extreme categories, since an action can potentially leak information regarding, not only its own reward distribution, but also the distributions of other actions. Drawing an example from music recommendation, the fact that a user likes a specific recommended song (corresponding to an action chosen by the platform) can potentially imply that similar songs (e.g., of the same artist or genre) might also be appreciated by the user. In another example, encouraging a certain user of a social network to advertise a purchased product (potentially by offering a discount as part of a promotion campaign), the seller can obtain valuable information about other neighboring users by recording their reactions.

Motivated by such practical scenarios, researchers have focused their attention on online learning models with richer information structures -- namely, feedback which interpolates between full and bandit \cite{Mannor2011,Caron2012,lin2014,yun2018,Cortes2020}. While the presence of side observations permits improved regret guarantees compared to those of standard multi-armed bandits, techniques and algorithms that have been proved successful for bandit-feedback (e.g., optimism principle and the UCB method) fail to achieve optimality. Indeed, enriching the feedback model severely perplexes the role of exploration and exploitation. 
To illustrate, one can think of the following example: $K-1$ arms provide to the player standard bandit feedback. In addition to these, there also exists an ``information-revealing'' arm with (deterministically) zero reward, such that, once played, the player gets full-feedback on the realized rewards of the round, for all arms. Even in the above simple example, the number of times (if any) this ``information-revealing'' arm should played in a minimum regret solution is now a complex function of $K$ and the (a priori unknown) suboptimality gaps.

Initiated by the work of \cite{Mannor2011} in the context of adversarial bandits, the majority of works in this field focuses on the so-called {\em graph-structured feedback} model \cite{noga2015, wu2015, kocak2016, Arora2019, Cortes2020}. There, each arm corresponds to a node of a given (directed) graph and every time the player chooses an action, she observes the reward realization of the arm played, but also the realizations of all the adjacent arms in the graph. In the above setting, a series of works  \cite{buccapatnam2014, noga2015,cohen2016, Lykouris2020} has provided algorithms with regret guarantees that no longer depend on the number of arms, but on smaller inherent characteristics of the feedback-graph (e.g., clique or independence number). 

One of the most general feedback models that has been introduced in the literature is due to Wu, Szepesv\'{a}ri, and Gy\"{o}rgy \cite{wu2015}. According to their model, playing an arm reveals noisy information about the reward of all other arms according to an arbitrary a priori known {\em side information} matrix. Specifically, each element of this square matrix encodes the fidelity of the information -- expressed in terms of standard deviation of the noise -- that the ``row" arm reveals about the ``column" arm. Notice that, modulo the Gaussian noise assumption made in \cite{wu2015}, the above model subsumes the full, bandit, and graph-structured feedback as special cases. 
For the case of Gaussian noise, \cite{wu2015} provide finite-time instance-dependent lower bounds on the regret of any algorithm. In terms of upper bounds, the authors only address the special case where each entry of the feedback matrix satisfies $\sigma_{i,j} \in \{\sigma, \infty\}$ for some fixed $\sigma$ (which is essentially the graph-structured feedback model), but they leave open the question of an algorithm for general feedback matrices.

\subsection{Our Contributions}
\label{subsection:Contribution}

In our work, we provide the first algorithm and regret analysis for the setting of general feedback matrices proposed by \cite{wu2015}. 
As we show, the regret guarantee of our algorithm is asymptotically optimal, as it matches our asymptotic instance-dependent lower bound. We now outline the main challenges and technical contributions of this work. 

\textbf{Concentration bounds for a natural ML estimator.} In the restricted case where the entries of the feedback matrix are in $\{\sigma, \infty\}$, the natural estimator of the mean reward of an arm is the sample average of the (finite-variance) samples collected. In our algorithm, where the collected samples are subject to different noise levels, we replace the above estimator by the weighted average of the collected samples, where the weight of each sample is the inverse variance of the noise of its source. The above maximum-likelihood estimator suggests that the notion of ``number of samples'', as a measure of the amount of information collected, needs to be replaced by the {\em weighted number of samples}.

The use of weighted average as an estimator introduces technical hurdles: in scenarios where the number of samples is a random variable that depends on the trajectory of the algorithm, one needs to apply a union bound over all possible sample numbers in order to decorrelate the estimate from the evolution of the algorithm up to each round. However, in our case where each sample can be collected under $K$ different noise levels, the weighted number of samples of an arm can take exponentially many different values -- a fact that invalidates the use of union bound. In \cref{section:Estimator}, we show how to overcome the above issue by proving concentration guarantees for our estimator, which extend to the more general case of sub-Gaussian noise.

\textbf{Asymptotic regret lower bound.} In our setting, a minimum regret arm-pulling schedule is a complicated function of the number of arms, the suboptimality gaps, and the feedback matrix. In the imaginary scenario where prior knowledge of the suboptimality gaps is assumed, the underlying (combinatorial) problem is to collect sufficient information in order to distinguish the optimal arm by paying the minimum cost (i.e., total suboptimality gap of the arms played for exploration). Although this problem is computationally hard, it can be closely approximated (up to rounding errors) and relaxed by a linear program (LP). As we prove in \cref{section:Asymptotic}, since the above approximation becomes tighter as the time horizon increases, the exact same LP provides a clear and intuitive asymptotic lower bound on the regret. 

\textbf{Asymptotically optimal algorithm.} The LP-based asymptotic regret lower bound that we construct serves as a starting point for the design of our asymptotically optimal algorithm. Since the suboptimality gaps are initially unknown, the role of our algorithm is now twofold: to estimate the LP (including constraints and objective) up to a sufficient degree, while simultaneously to implement its optimal solution in an online manner. In order to achieve the above dual objective, our algorithm interleaves rounds of pure greedy exploitation, when sufficient information (with respect to the estimated LP) has been collected, and exploration rounds. The latter case is further partitioned into rounds where, by greedily collecting high fidelity samples, the algorithm attempts to uniformly estimate the LP, and more refined exploration rounds that are dictated by its solution. In \cref{section:Algorithm}, we describe our algorithm and prove its asymptotic optimality.

Due to space constraints, all omitted proofs have been moved to the Appendix. 

\subsection{Related Work}
\label{subsection:Related}

The problem of bandits with graph-structured feedback was introduced by \cite{Mannor2011} in an adversarial setting. Their model naturally interpolates between experts, where the learner observes feedback from all actions, and bandits, where the learner receives feedback only from the action selected at each round. Since then, the graph feedback model has been extensively studied in both adversarial and stochastic settings.

In the stochastic case, the work of \cite{Caron2012} presents a $\Omega(\log(T))$ regret lower bound for the graph feedback setting, as long as not all suboptimal arms have a maximum-reward neighbor. 
In addition, \cite{Caron2012} examine the regret of natural UCB variants that include the side observations in the UCB indices. They provide instance-dependent regret guarantees of worst-case order $\mathcal{O}\left(\chi\cdot \frac{\D_{\max}}{\D_{\min}^2}\cdot\log(T)\right)$, where $\chi$ is the \textit{clique-partition number} of the graph (i.e. the minimum number of cliques into which the graph can be partitioned). 
\cite{buccapatnam2014} provide an asymptotic instance-dependent LP lower bound for the graph-structured feedback model. 
The authors also propose two algorithms which exploit (a relaxation of) the minimum \textit{dominating set} of the graph (i.e., the smallest set of nodes which, in terms of their one-step neighbors, covers all arms) in order to perform more efficient exploration.
Their gap-dependent regret guarantees are of the form $\mathcal{O}\left(\sum_{i\in D} \frac{\log(T)}{\D_i}\right)$, where $D$ is a special \textit{dominating set} of the graph. Similar guarantees for UCB and Thomson Sampling are provided by \cite{Lykouris2020} in terms of the \textit{independence number} of the feedback graph (i.e. the size of the maximum independent set). 

In principle, it seems that the absence of initial knowledge of the suboptimality gaps prevents 
the above approaches from achieving optimal regret guarantees.
The work of \cite{wu2015} is the first to resolve the above issue and achieve optimal instance-dependent regret. Focusing on the Gaussian setting, they present an algorithm that attempts to estimate the gaps as well as satisfy the constraints of the asymptotic instance-dependent LP lower bound due to \cite{buccapatnam2014}. 
They also present a minimax optimal algorithm for this setting. 
In addition, \cite{wu2015} generalize the graph-structured feedback model to arbitrary feedback matrices, for which they present finite-time gap-dependent lower bounds. 

The graph-structured feedback model has also been extensively studied in an adversarial setting. \cite{Mannor2011} present algorithms with regret guarantees of the form $\mathcal{O}(\sqrt{\alpha\cdot \log(k) \cdot T})$, for
undirected, and $\mathcal{O}(\sqrt{\chi \cdot\log(k) \cdot T})$ for directed graphs, together with ${\Omega}(\sqrt{\alpha \cdot T})$ lower bounds. \cite{noga2013} provide improved results for the directed feedback graph case. In \cite{noga2015}, the authors investigate the relation between the observability properties of the graph and the optimal achievable regret. \cite{kocak2016} generalize the graph feedback to weighted graph and provide $\mathcal{\Tilde{O}}(\sqrt{\alpha^* \cdot T})$ guarantees, where $\alpha^*$ is a generalization of the independence number for weighted graphs and the $\tilde{\mathcal{O}}(\cdot)$ notation is used to hide logarithmic terms. Other variants of adversarial online learning with graph-structured feedback have been studied in \cite{Rangi2019,Cortes2019,Arora2019,Resler2019,Cortes2020}. 

Finally, additional examples of bandits that incorporate graph-structured feedback include the work of \cite{Liu2018,Liu2018b}, where the authors study the Bayesian version of the problem and provide $\mathcal{O}(\sqrt{\alpha\cdot T \cdot\log(K)})$  regret guarantees for time-varying graphs. In \cite{Liu2018} the authors employ information directed sampling for this setting. \cite{Li2020} extend the graph-structured feedback by adding observation probabilities on the edges. \cite{cohen2016} study the case where the feedback graph is never fully revealed to the learner. Other variations of the problem include \cite{yun2018,singh2020, wang2020}, and the so-called partial monitoring setting \cite{lin2014,Bartok2014,Hanawal2016}.

\section{Problem Definition}\label{section:Model}

\paragraph{Model.} We consider the variant of the stochastic multi-armed bandit problem where the player is given $K$ arms with (unknown) expected rewards $\mu = (\mu_1,\ldots,\mu_K)$ and a feedback matrix $\Feedback = (\sigma_{i,j})_{i,j\in[K]}$. By playing an action $i \in [K]$ the player observes a noisy sample $X_{j}$ from the reward of each arm $j \in [K]$, distributed independently as $X_j \sim \mathcal{N}(\mu_j,\sigma_{i,j}^2)$, and collects the realized value $X_i \sim \mathcal{N}(\mu_i,\sigma_{i,i}^2)$. Therefore, the matrix $\Feedback$ quantifies the quality of the noisy observations of the expected rewards in terms of standard deviation of the Gaussian noise. At any round $t$ where arm $i_t$ is played, we denote by $X_{j,t}$ the noisy observation for each arm $j \in [K]$, which coincides with the realized collected reward in the case where $j = i_t$. The objective is to minimize the expected cumulative regret over an unknown time horizon $T$, defined as 
\begin{align*}
    R_T(\mu) = T \cdot \mu^*  -  \Ex{}{\sum_{t\in[T]} X_{i_t,t}}, 
\end{align*}
where $\mu^* = \max_{i\in[K]}\mu_i$ is the maximum expected reward.

We remark that we do not impose any non-trivial restrictions on the feedback matrix $\Feedback \in \mathbb{R}^{k \times k}_{\geq 0}$. In particular, we allow $\Feedback$ to be asymmetric, i.e., $\sigma_{i,j} \neq \sigma_{j,i}$ for $i \neq j$, and we permit that $\sigma_{j,i} < \sigma_{i,i}$, namely, higher quality information about the reward of an action can potentially be obtained by playing a different action. Finally, $\s_{i,j} = \infty$ corresponds to the case where pulling arm $i$ reveals no information about the reward of arm $j$. 

\textbf{Technical notation.} We denote by $i^*(\mu)$ the maximum expected reward arm of vector $\mu$ and by $\mu^*$ its expected reward. In the case of more than one optimal arms, we choose $i^*(\mu)$ to be the optimal arm of smallest index. For any vector $\mu$, the suboptimality gap of arm $i\in[K]$ is defined as $\D_i(\mu)=\mu^* - \mu_i$. We define the minimum and maximum suboptimality gaps as $\D_{\min}(\mu) = \min_{i:\D_i(\mu)>0}\D_i(\mu)$ and $\D_{\max}(\mu)=\max_{i\in[K]}\D_i(\mu)$, respectively. For brevity, we simply use $\D_{\max}$ instead of $\D_{\max}(\mu)$ when $\mu$ is the actual mean reward vector.

For any feedback-matrix $\Feedback$, we define $\sigma^{\min}_i = \min_{j \in [K]} \sigma_{j,i}$ to be the minimum standard deviation of the noise under which we can obtain information about $\mu_i$. Note that, although the condition $\sigma^{\min}_i < \infty$ for any $i \in [K]$ is essential for an instance to be identifiable (indeed, if $\sigma^{\min}_i = \infty$ for some $i \in [K]$, then it is impossible to estimate $\mu_i$), we do not explicitly make this assumption, given that the dependence on $\{\sigma^{\min}_i\}_{i \in [K]}$ appears naturally in our results. Finally, we define $\sm = \max_{i \in [K]} \sigma^{\min}_i$ to be the maximum $\sigma^{\min}_i$ over all arms.

At any round $t$, we denote by $N_i(t)$ the number of times arm $i$ has been pulled up to and including round $t-1$. Finally, we define $\nt_i(t) = \sum_{j \in [K]} N_j(t) / \sigma^2_{j,i}$ to be the \textit{weighted number of samples} corresponding to arm $i \in [K]$ at the beginning of round $t$.

\section{Maximum-Likelihood Estimator and Concentration Bounds} 
\label{section:Estimator}
By definition of our model, every time an arm $j \in [K]$ is played at some round $\tau$, for each arm $i \in [K]$, the player observes a sample $X_{i,\tau}$ drawn independently from $\mathcal{N}(\mu_i,\sigma_{j,i}^2)$. From the perspective of a fixed arm $i \in [K]$, the player collects, at each round, a noisy estimate of $\mu_i$ under various (Gaussian) noise levels. In particular, the noise variance depends on the played arm, which in turn is a function of the trajectory of the algorithm up to that round. The results of the following sections rely on a natural maximum-likelihood (ML) estimator for the mean of samples from heterogeneous sources of identical means and different standard deviations, given by matrix $\Feedback$. In this section, we define this estimator and prove useful concentration properties. 
For brevity, for the rest of this section we fix an arm $i \in [K]$ and drop any reference to it. 

The ML estimator for the mean of any arm at the beginning of round $t$ (namely, using $t-1$ samples) is defined as follows:
\begin{align}\label{eq:estimator}
    \mh(t)  &= \sum_{\tau=1}^{t-1} \frac{X_{\tau}}{\s_{i_\tau}^2}  \left(\sum_{\tau=1}^{t-1} \frac{1}{\s_{i_\tau}^2}\right)^{-1} =\sum_{j\in[K]} \sum_{\tau=1}^{t-1} \frac{X_{\tau} \cdot \event{i_\tau=j}}{\s_{j}^2}   \nt(t)^{-1}, 
\end{align}
where, again, we use the definition of the weighted number of samples $\nt(t) = \sum_{j\in[K]} \frac{N_j(t)}{\s_{j}^2}$. 

Observe that the above estimator is a weighted average of $t-1$ samples coming from $K$ possible different sources (namely, arms played), where each sample is weighted by the inverse variance of its source. Here, the weighted number of samples $\nt(t)$ reflects the total ``quality'' of the information collected up to time $t$, and generalizes the ``number of collected samples'' -- a critical notion in the standard concentration results used in multi-armed bandits.

In the case where the number of samples from each type of noise distribution is fixed, the following guarantee can be trivially proved:  

\begin{fact}\label{rem:concentration} Assuming that the trajectory of the algorithm (namely, the type of sources) up to and including time $t-1$ is fixed and independent of the observed samples -- thus, $\nt(t)$ is deterministic -- the following inequality is true for any $\epsilon>0$ and $i \in [K]$:
\begin{align*}
    \Pr[|\mh(t) - \mu| > \epsilon]
    < 2 \cdot \exp{\left(-\frac{\nt(t) \cdot \epsilon^2}{2}\right)}.
\end{align*} 
\end{fact}

In scenarios of active sampling, however, where the actions depend on the history of observed rewards, the quantity $\nt(t)$ is also a random variable depending on the trajectory up to time $t-1$ -- a fact that invalidates the use of the above standard concentration result. 
\paragraph{Remark.} A usual approach in such settings is to apply a union bound over all possible numbers of samples, bypassing in that way the dependence between the trajectory and the observed samples. However, given that in our case $\nt(t)$ can take $\binom{t+K-2}{K-1}$ different values (since the reward of each arm can be sampled in $K$ different ways at each round), this approach would result in an additional $\mathcal{O}(t^K)$-factor with an undesirable exponential dependence on $K$ in the bound of \cref{rem:concentration}. 

In order to overcome the above issue, we first prove an auxiliary result, which includes concentration bounds for active sampling related to the estimator in \cref{eq:estimator}. 

\begin{restatable}{lemma}{lemmaSupermartingale}\label{lemma:martingale}
Let $\{Z_{t'}\}_{t' \in \mathbb{N}}$ be a sequence of random variables. We denote by $\mathcal{F}_{t'}$ the $\sigma$-algebra generated by $\{Z_{\tau}\}_{\tau \leq t'}$ and by $\mathcal{F} = (\mathcal{F}_{t'})_{{t'} \in \mathbb{N}}$ the corresponding filtration. Each random variable $Z_{t'}$ is drawn independently from a zero-mean sub-Gaussian distribution with variance proxy $\sigma_{t'}^2$, where $\sigma_{t'}$ is an $\mathcal{F}_{t'-1}$-measurable random variable. We define $W_{t'} = \sum^{t'}_{\tau = 1} \frac{Z_{\tau}}{\sigma_{\tau}^2}$ and $\nt_{t'} = \sum^{t'}_{\tau = 1} \frac{1}{\sigma^2_{\tau}}$. 
\begin{enumerate}
    \item[\textbf{(a)}] Let $\phi$ be an $\mathcal{F}$-stopping time which satisfies either $\nt_{\phi} \in I$ for some interval $I = [L , H]$ with $H>L>0$, or $\phi = t+1$. Then, we have that
\begin{align*}
    \Pr\bigg[\left| W_{\phi} \right| > \sqrt{2\alpha \nt_{\phi} \log{t}} \text{ and }\phi \leq t \bigg] \leq 2 \cdot t^{- \alpha {L}/{H}}.
\end{align*}
    
    \item[\textbf{(b)}] Let $\psi$ be an $\mathcal{F}$-stopping time which satisfies either $\nt_\psi \geq r$ for some $r \in \mathbb{R}_{\geq 0}$, or $\psi = t + 1$. Then, for any $\epsilon>0$, we have that
\begin{align*}
    \Pr\bigg[|W_{\psi}| > \nt_\psi \epsilon \text{ and }\psi \leq t \bigg] \leq 2 \cdot  \exp\left(-
    \frac{r \cdot \epsilon^2}{2} \right).
\end{align*}
\end{enumerate}
\end{restatable}
\begin{proof}[Proof sketch.]
As an auxiliary result, we first prove that the sequence $(\tilde{G}_{t'})_{t' \in \mathbb{N}}$, where $\tilde{G}_{t'} = \exp\left(\lambda W_{t'} - \frac{\lambda^2 \nt_{t'}}{2} \right) \cdot \event{t' \leq t}$ for some $\lambda \in \mathbb{R}$, is a super-martingale which satisfies $\Ex{}{\tilde{G}_{t'}} \leq 1$ for any $t' \in \mathbb{N}$. For proving part (a), we focus on bounding the probability that $W_{\phi} > \sqrt{2\alpha \nt_{\phi} \log{t}}$ and $\phi \leq t$, since the other tail bound follows symmetrically. 
By denoting $G_{t'} = \exp\left(\lambda(W_{t'} - \sqrt{2\alpha \nt_{t'} \log{t}}) \right) \cdot \event{t' \leq t}$ for some $\lambda > 0$, and using Markov's inequality, we get that
\begin{align*}
    \Pr\bigg[W_{\phi} > \sqrt{2\alpha \nt_{\phi} \log{t}} \text{ and }\phi \leq t \bigg] \leq \Ex{}{G_{\phi}}.
\end{align*}
In order to upper-bound $\Ex{}{G_{\phi}}$, by setting $\tilde{G}_{t'} = \exp\left(\lambda W_{t'} - \frac{\lambda^2 \nt_{t'}}{2} \right) \cdot \event{t' \leq t}$, we first rewrite 
\begin{align*}
G_{t'} = \tilde{G}_{t'} \cdot \exp\left(\frac{\lambda^2 \nt_{t'}}{2} - \lambda \sqrt{2 \alpha \nt_{t'} \log t} \right).
\end{align*}
Note that, for $t' = \phi$, the event that $\phi \leq t$ implies that $\nt_\phi \in I$ and, thus, $L \leq \nt_{\phi} \leq H$. Therefore, by setting $\lambda = \frac{1}{H} \sqrt{2 \alpha L \log t}$, $G_{\phi}$ can be upper-bounded as
\begin{align*}
G_{\phi} \leq \tilde{G}_{\phi} \cdot \exp\left(- \frac{\alpha \cdot L}{H} \log t \right).
\end{align*}
The proof follows by using the fact that $\Ex{}{\tilde{G}_{\phi}} \leq 1$ for any $\lambda > 0$.

The proof of (b) follows by similar arguments. 
\end{proof}

In the next Lemma, we provide a concentration result for our estimator in the case where the weighted number of samples randomly depends on the trajectory of observed realizations. The result holds in the case where we have at least one sample from the source with minimum noise. 

\begin{restatable}{lemma}{lemmaAnytime}\label{lemma:anytime}
Let $\alpha>0$. For any $t \geq 2$, for the estimator defined in \cref{eq:estimator}, 
if $\nt(t) \geq \frac{1}{(\s^{\min})^2}$, where $\s^{\min}=\min_{j\in[K]} \s_{j}$, then we have that
\begin{align*}
    \Pr\left[|\mh(t) - \mu| > \sqrt{\frac{2\alpha\log{t}}{\nt(t)}} \right] < 2 \cdot \lceil \log_2(t-1) \rceil \cdot t^{-\alpha/2}.
\end{align*}
\end{restatable}

Note that imposing conditions on $\nt(t)$ or the source noise is necessary in order to obtain any concentration results for the estimator in \cref{eq:estimator}: for instance, we cannot hope for concentration results in the case where all $t$ samples come from a source with $\s =\infty$ noise. The proof of the above Lemma is based on an exponentially-spaced discretization of the possible range of values of $\nt(t)$, combined with \cref{lemma:martingale}. Finally, we remark that another possible approach for obtaining concentration guarantees for our estimator is through the idea of the self-normalizing bound \cite{NIPS2011_e1d5be1c} combined with the assumptions of \cref{lemma:anytime}. However, this approach comes at an expense of a
worse dependence on t.

\section{Asymptotic Regret: Lower Bound}\label{section:Asymptotic}
In this section, we provide an asymptotic regret lower bound for policies that are consistent with respect to any environment $(\mu, \Feedback)$. We recall the definition of {\em consistent} policies:
\begin{definition}\label{def:consistent}
A policy is called {\em consistent} if, for any environment $(\mu, \Feedback)$, its regret satisfies
\begin{align*}
    \lim_{T\rightarrow \infty}{\frac{R_T(\mu)}{T^p}} = 0 ~~~\text{  for any }p>0.
\end{align*}
\end{definition}

In this section, we show that, in the asymptotic regime, the problem of collecting sufficient information to distinguish the suboptimality gaps of the arms, while paying the minimum cost in terms of regret accumulated during suboptimal plays, can be closely approximated by an LP. 
For any reward vector $\mu$, we define the following parameterized set of constraints:
\begin{align}\label{eq:lpconstraints}
    C(\mu)=
    \begin{Bmatrix}
     &\sum_{j\in[K]} \frac{c_j}{\s_{j,i}^2} \geq \frac{2}{\D_i^2(\mu)},\forall i\not = i^*(\mu)\\
    {c\in\mathbb{R}_{\geq 0}^K :} & \\
     & \sum_{j\in[K]} \frac{c_j}{\s_{j,i}^2} \geq \frac{2}{\D_{\min}^2(\mu)},i = i^*(\mu)
    \end{Bmatrix}.
\end{align}

To provide some intuition on the above constraint set, we recall the case of standard bandit feedback, where $\s_{i,i}<\infty$ and $\s_{i,j}=\infty$ for any $i\not = j$. There, for any suboptimal arm $i$ the corresponding constraint of $C(\mu)$ becomes $c_i\geq \frac{ {2\s_{i,i}^2}}{\D_i^2(\mu)}$. This matches the multiplicative factor in the existing lower bounds for the standard bandit feedback case, where each arm must be played at least $\frac{2\s_{i,i}^2}{\D_i^2(\mu)}\log(T)$ times. 
In our general feedback setting, for any arm $i$, in addition to the information collected by playing the arm itself, the constraints also take into account the information collected by any other arm $j \in [K] \setminus \{i\}$, weighted by its inverse variance.

A natural lower bound on the {\em number} of suboptimal plays of each arm can be constructed using the minimum-cost feasible vector in $C(\mu)$ with respect to the suboptimality gaps, formally defined as
\begin{align}\label{eq:LP} 
    c^*(\mu) =  \underset{c\in C(\mu)}{\text{argmin}} \sum_{i\in [K]}  c_i \D_i(\mu).
\end{align}

As we show in the following theorem, the optimal solution of the LP in \cref{eq:LP} asymptotically characterizes a lower bound on the number of plays of each suboptimal arm relative to $\log(T)$.
\begin{restatable}{theorem}{thmLowerBound}\label{thm:lower_bound}
For any environment $(\mu, \Feedback)$, the regret of any consistent policy within $T$ rounds satisfies
\begin{align*}
    \liminf_{T\rightarrow \infty}{\frac{R_T(\mu)}{\log{T}}} \geq \sum_{i\in [K]} c_i^*(\mu) \D_i(\mu).
\end{align*}
\end{restatable}

In order to prove \cref{thm:lower_bound}, first, we consider a reward vector $\mu$ and identify a suboptimal arm $k$ in $\mu$. For some $\epsilon>0$ we construct a vector $\mu'$ such that 
\begin{align} \label{eq:means}
    \mu'_i=
    \begin{cases}
    \mu_i, \text{ if }i\not = k\\
    \mu^{*} +\epsilon, \text{ if }i=k.
    \end{cases}
\end{align}

Then, in the following proposition, we decompose the KL-divergence of two distributions $\mathbb{P}, \mathbb{P}'$, each capturing the interplay of a policy with environments $(\mu, \Feedback)$ and $(\mu', \Feedback)$, respectively.

\begin{restatable}{proposition}{propDivergence}\label{prop:divergence}
Let two Gaussian K-armed bandit instances $\nu$ and $\nu'$ with the same side information matrix $\Feedback$ and mean reward vectors $\mu$ and $\mu'$, respectively. Let $\mathbb{P}$ (resp. $\mathbb{P}'$) be the distribution associated with the interplay of $\nu$ (resp. $\nu'$) and a policy $\pi$ within $t$ rounds. If $\mu$ and $\mu'$ differ only in the reward of arm $k$, then the KL-divergence of $\mathbb{P}$ with respect to $\mathbb{P}'$ satisfies: 
\begin{align}
    D(\mathbb{P}~||~\mathbb{P}') = \sum_{i\in[K]} \Ex{\nu}{N_i(t)} \frac{(\mu_k-\mu_k')^2}{2\s_{i,k}^2}.
\end{align}
\end{restatable}

\cref{thm:lower_bound} follows by \cref{prop:divergence},  \cref{def:consistent} and an application of Bretagnolle-Huber inequality for distributions with mean rewards $\mu,\mu'$ as defined in \cref{eq:means}.

Before we proceed to the presentation of our algorithm, we comment that any given solution $c^*(\mu)$ of the LP in \cref{eq:LP} could be turned into an optimal Explore-Then-Commit (ETC) strategy for our problem: by playing each arm $\lceil c^*_i(\mu) \cdot \log T\rceil$ times, thus incurring regret at most $\sum_{i \in [K]} c^*_i(\mu) \Delta_i(\mu) \log T + \sum_{i \in [K]} \Delta_i(\mu)$, by \cref{lemma:anytime} we have collected enough information to distinguish the optimal arm with high probability. Hence, $c^*(\mu)$ readily describes an optimal (up to rounding errors) arm-pulling schedule with respect to the lower bound.

\section{Algorithm and Regret Analysis}\label{section:Algorithm}
In this section we provide the first algorithm for the general setting of Gaussian bandits with arbitrary side information matrix $\Feedback$ and prove its asymptotic optimality. 

\subsection{Description of the Algorithm and Main Results}

The general idea behind our algorithm is the following: recall that in the ideal scenario where the suboptimality gaps are known a priori, the linear program described in \cref{eq:LP} would directly provide an optimal (up to vanishing rounding errors) exploration strategy, namely, a minimum-regret arm-pulling schedule collecting the necessary information to distinguish the optimal arm with high probability. The above intuition, however, cannot be readily transformed into an algorithm, since prior knowledge of the suboptimality gaps would trivialize our problem. 

At a high level, our algorithm attempts to estimate the LP (including constraints and objective) described in  up to some accuracy, while at the same time implement its solution (i.e., play, for each arm, the indicated number of samples) in an online manner. The choice of the desired accuracy, up to which the LP must be estimated, is a key for the success of our algorithm: the LP needs to be estimated well-enough to allow the learner to take near-optimal decisions, yet without suffering a significant estimation overhead. In principle, our algorithm generalizes and extends that of \cite{wu2015} for the simple case of graph-structured feedback, i.e., when each entry of $\Feedback$ satisfies $\s_{i,j}\in \{\sigma,\infty\}$ for some $\sigma < \infty$.

Our algorithm is presented in pseudocode in \cref{algorithm1}. The algorithm starts by playing, for each arm $j \in [K]$, the arm $i_j \in [K]$ that provides the most accurate information about $j$. For the rest of the rounds, the algorithm either exploits the already collected information, or further explores the environment.

\paragraph{Greedy exploitation.} Let $\mh(t) \in \mathbb{R}^K$ be the vector of estimated means, where the coefficient $\mh_i(t)$ for every $i \in [K]$ is given in \cref{eq:estimator} (note that in \Cref{section:Estimator} we have dropped any reference to $i$ for generality). Recall that $C(\mh(t))$ is the constraint set of the estimated LP at time $t$, constructed using the collected samples. In the case where the collected number of samples for each arm provides (up to proper scaling) a feasible solution to $C(\mh(t))$ (see Case (A) of \Cref{algorithm1}), the algorithm greedily plays the arm with the maximum empirical mean reward. 

\paragraph{Uniform and LP-dictated exploration.}
If the collected samples are not feasible for $C(\mh(t))$ (up to proper scaling), then the algorithm enters an exploration phase where it either attempts to refine the estimate of the LP, called, or to make progress in exploring arms according to what the currently estimated LP dictates. We refer to each case as {\em Uniform} and {\em LP-dictated} exploration, respectively. Here, the player needs to address the delicate issue that the total exploration cost needs to be close to that of the optimal solution of the LP. This is achieved by performing a relatively small number of uniform exploration rounds, which allow the learner to obtain a sufficiently good estimate of the LP.

The algorithm considers the LP to be sufficiently (uniformly) explored if the weighted number of samples for each arm satisfies a uniform lower bound. In particular, our algorithm compares the minimum weighted number of samples over all arms with a sublinear function of the total sample weight collected during all exploration rounds. Formally, \cref{algorithm1} inspects the weighted number of samples of each arm, and checks whether $\min_{i\in [K]} \nt_i(t) < \frac{1}{K}\beta\left({n_e(t)}\right)$, where $\beta(x)=\frac{x^{\gamma}}{2\sm^{2}}$, with $\gamma\in (0,1)$ and $\sm = \max_{i \in [K]} \sigma^{\min}_i$, and $n_e(t)$ is the total number of exploration rounds up to time $t$. In the case where the arms are not uniformly explored (see Case (B) of \Cref{algorithm1}), our algorithm attempts to satisfy this uniform exploration bound, by playing the arm which provides the less noisy information about an arm $i$ that has the minimum $\nt_i(t)$ (breaking ties arbitrarily). 

On the other hand, if the arms are sufficiently uniformly explored (see Case (C) of \Cref{algorithm1}), our algorithm computes an optimal solution $c^*(\mh(t))$ to the estimated LP, and then plays any arm $i$ such that $N_i(t) < 4\alpha \cdot c^*_i(\mh(t))\log{t}$, namely, which is considered unexplored (again up to proper scaling) according to the LP solution.

\begin{algorithm2e} \label{algorithm1}
\DontPrintSemicolon
\caption{Asymptotically-Optimal Algorithm for Gaussian Bandits with Side Observations}
    {\bfseries Input:} $K$ arms, $\Sigma$, $\beta(x)=\frac{x^{\gamma}}{2\sm^{2}}, \gamma\in(0,1)$, $\alpha>4$\;
    For each arm $j \in [K]$, play arm $i_j \gets \arg\min_{i\in[K]}{\s_{i,j}^2}$\;
    Set $n_e(K+1) \gets 0$\;
    \For{$t=K+1 , K+2, \ldots$}{
    \If{$\left(\frac{N_1(t)}{4a\log{t}},...,\frac{N_K(t)}{4a\log{t}}\right) \in C(\mh(t))$}
    {\tcc{\underline{Case A}: Greedy exploitation (event $\Vt$ in the analysis)~~~~~~~~~} 
    Play arm $i_t \gets \arg\max_{i\in[K]}\mh_i(t)$\;
    Set $n_e(t+1) \gets n_e(t)$\;
    }
    \ElseIf{$\min_{i\in [K]} \nt_i(t) < \frac{1}{K}\beta\left({n_e(t)}\right)$}{
    \tcc{\underline{Case B}: Uniform exploration (event $\Vt^c, \Wt$ in the analysis)~~~~~}
    Let $j \gets \arg\min_{k\in[K]}\nt_k(t)$\;
    Play $i_t \gets \arg\min_{k\in [K]}{\s_{k,i}^2}$\;
    Set $n_e(t+1) \gets n_e(t)+1$\;
    }
    \Else{\tcc{\underline{Case C}: LP-Dictated exploration (event $\Vt^c,\Wt^c$ in the analysis)}
    Compute $c^*(\mh(t))~\gets~\arg\min_{c\in C(\mh(t))} \sum_{i\in [K]}  c_i \D_i(\mh(t))$\;
    Play $i_t \gets i$ for some $i \in [K]$ such that $N_i(t) < 4\alpha \cdot c^*_i(\mh(t))\log{t}$\;
    Set $n_e(t+1) \gets n_e(t)+1$\;
    }
    }
\end{algorithm2e}


\noindent
\textbf{Main Results.} Before we state the regret guarantee of \cref{algorithm1}, we introduce some useful definitions. For any vector $\mu$, we consider the family of vectors $\mu'$ that are guaranteed to be $\epsilon$-close to $\mu$ with respect to the $\ell_{\infty}$-norm, that is, $|\mu'_i-\mu_i|\leq \epsilon$ for all arms $i \in [K]$. Let $c^*(\mu')$ be the optimal solution of the minimization problem in \cref{eq:LP} using parameters $\mu'$. The following quantity appears naturally in our regret guarantees: 

\begin{definition} \label{def:epsilonLP}
For any vector $\mu$, the worst case $\epsilon$-approximate LP solution for arm $j \in [K]$ is defined as
\begin{align*}
    c^*_j(\mu,\epsilon) = \sup_{\mu':\|\mu'-\mu\|_{\infty} \leq \epsilon} c^*_j(\mu'). 
\end{align*}
\end{definition}
\noindent
Note that, by continuity, taking $\epsilon\rightarrow 0$ we get $c^*_j(\mu,\epsilon) \rightarrow c^*_j(\mu)$ for every arm $j \in [K]$. 
We prove the following regret upper bound for \Cref{algorithm1}:

\begin{restatable}{theorem}{thmRegretUpperBound}\label{thm:regret_upper_bound}
For any $\alpha>4$, $\gamma \in (0,1)$, and $\epsilon>0$, the regret of algorithm \cref{algorithm1} satisfies
\begin{align*}
    R_T(\mu)
    &\leq \left(2K + \frac{8K}{\alpha-4}  + 2\right)\D_{\max} ~+~ 2\D_{\max} K \sum_{\tau\in[T]} \exp\left(-\frac{\tau^{\gamma} \epsilon^2}{4K\sm^2}\right) \\
    &\qquad + \D_{\max} \left(4\alpha \sum_{j\in[K]} c^*_j(\mu,\epsilon)  \log{T}  +K \right)^{\gamma} 
    ~+~ 4\alpha \sum_{j\in[K]} \D_j(\mu) c^*_j(\mu,\epsilon)  \log{T}.
\end{align*}
\end{restatable}
We remark that, in the restricting case of graph side-information structure (i.e. each $\s_{i,j}=\sigma$ or $\infty$), the regret upper bound of \Cref{algorithm1} also matches the asymptotically-optimal regret shown in \cite{wu2015}. Further, notice that, as $\sm = \max_{i \in [K]} \sigma^{\min}_i$ grows to infinity, the dependence on $T$ in the second term of the above bound becomes linear. This is a natural effect, however, since having a large $\sm$ implies the existence of an arm $i \in [K]$ with large $\sigma^{\min}_i$, for which there is no way of getting accurate information.

The following corollary bounds the asymptotic regret of \Cref{algorithm1} in terms of the worst case $\epsilon$-approximate LP solution. 

\begin{restatable}{corollary}{corAsymptoticOptimality}\label{cor:asymptotic_optimality}
As the time horizon tends to infinity, the regret of \cref{algorithm1} satisfies
\begin{align*}
    \limsup_{T\rightarrow\infty} \frac{R_T(\mu)}{\log{T}} \leq 4\alpha \sum_{j\in[K]} \D_j(\mu) c^*_j(\mu,\epsilon).
\end{align*}
\end{restatable}

Notice that, although the above bound matches the lower bound of \cref{thm:lower_bound} for $\epsilon = 0$, we are not allowed to directly use $\epsilon \rightarrow 0$ in the bound of \cref{thm:regret_upper_bound}, since that would make the term $\sum_{\tau\in[T]} \exp\left(-\frac{\tau^{\gamma} \epsilon^2}{4K\sm^2}\right)$ linear in $T$. However, by choosing $\epsilon$ to be a decreasing function of $T$, e.g. $\epsilon \sim \log(T)^{-\gamma/4}$,  the suboptimal terms in \cref{thm:regret_upper_bound} are vanishing as $T\rightarrow \infty$ and the regret bound in \cref{cor:asymptotic_optimality} matches the asymptotic regret lower bound (up to constant factors).

\subsection{Outline of the Regret Analysis}
For the rest of this section, we provide an overview of our regret analysis, while the complete proof can be found in \cref{appendix:upper_bound}.

Clearly, the first $K$ rounds contribute at most a $(K-1) \cdot \D_{\max}$-additive loss in the regret. For each round $t \geq K+1$, we consider the following events:
\begin{align*}
    &\Vt = \left\{\left(\frac{N_1(t)}{4a\log{t}}, \ldots ,\frac{N_K(t)}{4a\log{t}}\right) \in C(\mh(t))\right\}, \\
    &\Wt = \left\{\min_{i\in [K]} \nt_i(t) < \frac{1}{K}\beta(n_e(t)) \right\}.
\end{align*}
Notice that the above events immediately characterize how \cref{algorithm1} operates at each round. Specifically, for $t \geq K+1$, the algorithm enters in Case A when $\Vt$ holds, in Case B when $\Vt^c, \Wt$, and in Case C, when $\Vt^c, \Wt^c$.

We also define the error event:
\begin{align*}
    \Ut = \left\{|\mh_i(t) - \mu_i|\leq \sqrt{\frac{2a\log{t}}{\nt_i(t)}}, ~~\forall i\in [K]\right\}.
\end{align*}
As we show in \cref{lemma:anytime} of \cref{section:Estimator}, for any $t$ the event $\Ut$ holds with high probability. Therefore, we can assume that $\Ut$ holds in every round, since the probability that $\Ut$ does not hold at some round converges to a small $\mathcal{O}(\frac{K}{\alpha - 4} \cdot \D_{\max})$-additive loss in the regret.

\paragraph{Rounds satisfying~~$\Ut,\Vt$.} When the event $\Ut$ holds, if the algorithm enters Case A, it can be shown that it chooses the optimal arm during the greedy selection. In particular, the event $\Ut,\Vt$ implies that for any arm $i$ that is suboptimal with respect to $\mh(t)$, that is, $i\not = i^*(\mh(t))$, we have that
\begin{align*}
    \mu_i - \mh_i(t) \leq \frac{\D_i(\mh(t))}{2},
\end{align*}
while for the optimal arm $i^*(\mh(t))$, we get 
\begin{align*}
    \mh_{{i^*(\mh(t))}}(t) - \mu_{i^*(\mh(t))} \leq \frac{\D_{\min}(\mh(t))}{2}.
\end{align*}
Since each true parameter $\mu_i$ is $\frac{\D_i(\mh(t))}{2}$-close to the corresponding estimate at time $t$, we can conclude that the arm selected by \cref{algorithm1} in that case is the optimal. Hence, the rounds such that $\Ut,\Vt$ holds do not contribute to the regret.

\paragraph{Rounds satisfying~~$\Ut,\Vt^c, \Wt$.}
We now turn our focus to the rounds where $\Vt^c,\Wt$ holds, which implies that \cref{algorithm1} enters Case B. 

Through a counting argument we show that the above can happen at most $\sm^2 \beta(n_e)+1$ times, where $n_e$ is the total number of exploration rounds, i.e. $n_e = \sum_{t=K+1}^T \event{\Vt^c}$. Specifically, we show the following:

\begin{restatable}{lemma}{lemmaCounting}\label{lemma:counting}
The following inequality holds:
\begin{align*}
    \sum_{t=K+1}^{T} \event{\Vt^c,\Wt} &\leq 
    \frac{1}{2}\left(\sum_{t=K+1}^T \event{\Vt^c}\right)^{\gamma} + 1.
\end{align*}
\end{restatable}
Notice that the fact that our algorithm plays the most informative arm (i.e., the one with smaller noise) in order to collect a sample for arm $i = \arg\min_{k\in[K]}\nt_k(t)$ is crucial for the above counting argument to hold.  
Thus, using \cref{lemma:counting}, we have that
\begin{align}
    \sum_{t=K+1}^T \event{\Ut,\Vt^c,\Wt} 
    &\leq \sum_{t=K+1}^T \event{\Vt^c,\Wt} \leq \frac{\left(\sum_{t=K+1}^T \event{\Vt^c}\right)^{\gamma}}{2} + 1 \label{eq:sketch1}.
\end{align}
We now bound the regret accumulated during all exploration rounds (that is, including Cases (2) and (3)). We have that
\begin{align}
    \event{\Vt^c} \leq &\event{\Ut^c}
    + \event{\Ut,\Vt^c,\Wt^c} +\event{\Ut,\Vt^c,\Wt} \label{eq:sketch2}.
\end{align}
By combining \cref{eq:sketch1} with \cref{eq:sketch2}, we can see that in order to upper bound $$\sum_{t=K+1}^T \event{\Ut,\Vt^c,\Wt},$$ it suffices to provide a bound on $\sum_{t=K+1}^T \event{\Ut,\Vt^c,\Wt^c}$. By doing this we are able to conclude that, overall, the number of rounds such that $\Ut,\Vt^c, \Wt$ holds is at most order 
\begin{align*}
    &K\sum_{\tau\in[T]}  \exp\left(-\frac{\tau^{\gamma} \epsilon^2}{4K\sm^2}\right) + \left(\alpha \sum_{j\in[K]} c^*_j(\mu,\epsilon)  \log{T} +K \right)^{\gamma}.
\end{align*}

\paragraph{Rounds satisfying~~$\Ut,\Vt^c, \Wt^c$.}

It remains to describe how we bound the regret accumulated in rounds where $\{\Vt^c,\Wt^c\}$ holds, and thus the algorithm enters Case C. 

We define the following event which states that, at time $t$, the error in the mean estimates upper bounded by $\epsilon$:
\begin{align}
    \Yt = \left\{|\mh_i(t) - \mu_i|\leq \epsilon, ~~\forall i\in [K]\right\}.
\end{align}
In order to bound the error of case $\{\Ut, \Vt^c,\Wt^c\}$, we further distinguish two cases according to whether $\Yt$ holds: 

Recall that, when $\{\Vt^c,\Wt^c\}$ holds, the algorithm selects an arm $j$ such that $\frac{N_j(t)}{4a\log{t}} < c^*_j(\mh(t))$. For rounds where $\{\Ut, \Vt^c,\Wt^c, \Yt\}$ holds, by \cref{def:epsilonLP} in combination with the definition of $\Yt$, we prove the following result:
\begin{restatable}{lemma}{lemmaDominantTerm}\label{lem:leading_term}
For every arm $j \in [K]$, it holds that
\begin{align*}
    \sum_{t=K+1}^T \event{\Ut,\Vt^c,\Wt^c,\Yt, i_t=j} 
    \leq 4\alpha \cdot c^*_j(\mu,\epsilon)  \log{T}.
\end{align*}
\end{restatable}

The above results immediately implies that the contribution to the regret of rounds $t \geq K+1$ such that $\{\Ut, \Vt^c,\Wt^c, \Yt\}$ is at most 
$4\alpha \sum_{j\in[K]} \D_j(\mu) c^*_j(\mu,\epsilon)  \log{T}$.

Finally, for the rounds such that $\{\Ut, \Vt^c,\Wt^c, \Yt^c\}$ holds, i.e., when the error in some of the estimates is greater than $\epsilon$, by applying a union bound over all arms, we get
{\small
\begin{align*}
    \event{\Ut,\Vt^c,\Wt^c,\Yt^c} \leq  \sum_{i\in[K]}\event{\Ut,\Vt^c,\min_{j\in [K]} \nt_j(t) \geq \frac{\beta(n_e(t))}{K}, |\mh_i(t) - \mu_i|> \epsilon}. 
\end{align*}
}
Given that $\min_{j \in [K]} \nt_j(t) \geq \frac{\beta(n_e(t))}{K}$ implies that $\nt_i(t) \geq \frac{\beta(n_e(t))}{K}$ for any arm $i \in [K]$, using part (c) of \cref{lemma:martingale}, we can upper bound the probability of each term in the above summation by $\exp\left(-\frac{\epsilon^2}{2} \frac{1}{K}\beta(\tau)\right)$.
Thus, the total contribution to the regret of the rounds such that $\{\Ut, \Vt^c,\Wt^c, \Yt^c\}$ holds can be upper bounded by $\D_{\max} K \sum_{\tau\in[T]}  \exp\left(-\frac{\tau^{\gamma} \epsilon^2}{4K\sm^2}\right)$.

The regret guarantee presented in \cref{thm:regret_upper_bound} follows by combining the above losses.

\section*{Conclusion and Further Directions}

In this work, we revisit the general feedback model introduced by Wu, Szepesv\'{a}ri, and Gy\"{o}rgy in \cite{wu2015} and provide the first algorithm for the case of arbitrary feedback matrices. To the best of our knowledge, this is one of the most general feedback models that appears in the literature, modulo the Gaussian noise assumption. 

A number of questions still remain open in the area of online learning under rich feedback structures. For instance, it would be interesting to explore the existence of algorithms that achieve (minimax) optimal regret in the finite time horizon regime. Another direction would be to examine different noise models, other than Gaussian, or even general sub-Gaussian noise with known variance proxies. We believe that our work could serve as a building block in the above direction. 

\section*{Acknowledgements}

This research  is supported in part by NSF Grants 2019844, 2107037 and 2112471, the Machine Learning Lab (MLL) at UT Austin, and the Wireless Networking and Communications Group (WNCG) Industrial Affiliates Program.



\newcommand{\etalchar}[1]{$^{#1}$}
\begin{thebibliography}{ACBGM13}

\bibitem[ACBDK15]{noga2015}
Noga Alon, Nicol{\`o} Cesa-Bianchi, Ofer Dekel, and Tomer Koren.
\newblock Online learning with feedback graphs: Beyond bandits.
\newblock In {\em COLT}, 2015.

\bibitem[ACBGM13]{noga2013}
Noga Alon, Nicolò Cesa-Bianchi, Claudio Gentile, and Yishay Mansour.
\newblock From bandits to experts: A tale of domination and independence.
\newblock {\em Advances in Neural Information Processing Systems}, 07 2013.

\bibitem[AMM19]{Arora2019}
Raman Arora, Teodor~Vanislavov Marinov, and Mehryar Mohri.
\newblock Bandits with feedback graphs and switching costs.
\newblock In {\em NeurIPS}, 2019.

\bibitem[AyPS11]{NIPS2011_e1d5be1c}
Yasin Abbasi-yadkori, D\'{a}vid P\'{a}l, and Csaba Szepesv\'{a}ri.
\newblock Improved algorithms for linear stochastic bandits.
\newblock In J.~Shawe-Taylor, R.~Zemel, P.~Bartlett, F.~Pereira, and K.Q.
  Weinberger, editors, {\em Advances in Neural Information Processing Systems},
  volume~24. Curran Associates, Inc., 2011.

\bibitem[BCB12]{Bubeck2012}
S{\'e}bastien Bubeck and Nicol{\`o} Cesa-Bianchi.
\newblock Regret analysis of stochastic and nonstochastic multi-armed bandit
  problems.
\newblock {\em Found. Trends Mach. Learn.}, 5:1--122, 2012.

\bibitem[BES14]{buccapatnam2014}
Swapna Buccapatnam, Atilla Eryilmaz, and Ness Shroff.
\newblock Stochastic bandits with side observations on networks.
\newblock {\em ACM SIGMETRICS Performance Evaluation Review}, 42, 06 2014.

\bibitem[BFP{\etalchar{+}}14]{Bartok2014}
Gábor Bartók, Dean Foster, Dávid Pál, Alexander Rakhlin, and Csaba
  Szepesvári.
\newblock Partial monitoring—classification, regret bounds, and algorithms.
\newblock {\em Mathematics of Operations Research}, 39:967--997, 11 2014.

\bibitem[CDG{\etalchar{+}}19]{Cortes2019}
Corinna Cortes, Giulia DeSalvo, Claudio Gentile, Mehryar Mohri, and Scott Yang.
\newblock Online learning with sleeping experts and feedback graphs.
\newblock In {\em ICML}, 2019.

\bibitem[CDG{\etalchar{+}}20]{Cortes2020}
Corinna Cortes, Giulia DeSalvo, Claudio Gentile, Mehryar Mohri, and Ningshan
  Zhang.
\newblock Online learning with dependent stochastic feedback graphs.
\newblock In {\em ICML}, 2020.

\bibitem[CHK16]{cohen2016}
Alon Cohen, Tamir Hazan, and Tomer Koren.
\newblock Online learning with feedback graphs without the graphs.
\newblock In Maria~Florina Balcan and Kilian~Q. Weinberger, editors, {\em
  Proceedings of The 33rd International Conference on Machine Learning},
  volume~48 of {\em Proceedings of Machine Learning Research}, pages 811--819,
  New York, New York, USA, 20--22 Jun 2016. PMLR.

\bibitem[CKLB12]{Caron2012}
St{\'e}phane Caron, Branislav Kveton, Marc Lelarge, and Smriti Bhagat.
\newblock Leveraging side observations in stochastic bandits.
\newblock In {\em UAI}, 2012.

\bibitem[HSS16]{Hanawal2016}
Manjesh~Kumar Hanawal, Csaba Szepesv{\'{a}}ri, and Venkatesh Saligrama.
\newblock Sequential learning without feedback.
\newblock {\em CoRR}, abs/1610.05394, 2016.

\bibitem[KNV16]{kocak2016}
Tom{\'a}s Koc{\'a}k, Gergely Neu, and Michal Valko.
\newblock Online learning with noisy side observations.
\newblock In {\em AISTATS}, 2016.

\bibitem[LAK{\etalchar{+}}14]{lin2014}
Tian Lin, Bruno Abrahao, Robert Kleinberg, John Lui, and Wei Chen.
\newblock Combinatorial partial monitoring game with linear feedback and its
  applications.
\newblock {\em 31st International Conference on Machine Learning, ICML 2014},
  3, 06 2014.

\bibitem[LBS18]{Liu2018}
Fang Liu, Swapna Buccapatnam, and Ness~B. Shroff.
\newblock Information directed sampling for stochastic bandits with graph
  feedback.
\newblock In {\em AAAI}, 2018.

\bibitem[LCWL20]{Li2020}
Shuai Li, Wei Chen, Zheng Wen, and Kwong-Sak Leung.
\newblock Stochastic online learning with probabilistic graph feedback.
\newblock {\em ArXiv}, abs/1903.01083, 2020.

\bibitem[LR85]{lai19854}
T.L Lai and Herbert Robbins.
\newblock Asymptotically efficient adaptive allocation rules.
\newblock {\em Advances in Applied Mathematics}, 6(1):4--22, 1985.

\bibitem[LTW20]{Lykouris2020}
Thodoris Lykouris, {\'E}va Tardos, and Drishti Wali.
\newblock Graph regret bounds for thompson sampling and ucb.
\newblock In {\em ALT}, 2020.

\bibitem[LZS18]{Liu2018b}
Fang Liu, Zizhan Zheng, and Ness~B. Shroff.
\newblock Analysis of thompson sampling for graphical bandits without the
  graphs.
\newblock {\em ArXiv}, abs/1805.08930, 2018.

\bibitem[MS11]{Mannor2011}
Shie Mannor and Ohad Shamir.
\newblock From bandits to experts: On the value of side-observations.
\newblock In {\em Proceedings of the 24th International Conference on Neural
  Information Processing Systems}, NIPS'11, page 684–692, Red Hook, NY, USA,
  2011. Curran Associates Inc.

\bibitem[RF19]{Rangi2019}
Anshuka Rangi and Massimo Franceschetti.
\newblock Online learning with feedback graphs and switching costs.
\newblock {\em AISTATS}, 2019.

\bibitem[RM19]{Resler2019}
Alon Resler and Y.~Mansour.
\newblock Adversarial online learning with noise.
\newblock In {\em ICML}, 2019.

\bibitem[SLLS20]{singh2020}
Rahul Singh, Fang Liu, Xin Liu, and Ness Shroff.
\newblock Contextual bandits with side-observations.
\newblock {\em ArXiv}, 2020.

\bibitem[Tsy08]{tsybakov2008}
Alexandre~B. Tsybakov.
\newblock {\em Introduction to Nonparametric Estimation}.
\newblock Springer Publishing Company, Incorporated, 1st edition, 2008.

\bibitem[WGS15]{wu2015}
Yifan Wu, Andr\'{a}s Gy\"{o}rgy, and Csaba Szepesv\'{a}ri.
\newblock Online learning with gaussian payoffs and side observations.
\newblock In {\em Proceedings of the 28th International Conference on Neural
  Information Processing Systems - Volume 1}, NIPS'15, page 1360–1368,
  Cambridge, MA, USA, 2015. MIT Press.

\bibitem[WLZ{\etalchar{+}}20]{wang2020}
Lingda Wang, Bingcong Li, Huozhi Zhou, Georgios~B. Giannakis, Lav~R. Varshney,
  and Zhizhen Zhao.
\newblock Adversarial linear contextual bandits with graph-structured side
  observations.
\newblock {\em CoRR}, abs/2012.05756, 2020.

\bibitem[YAP{\etalchar{+}}18]{yun2018}
Donggyu Yun, Sumyeong Ahn, Alexandre Proutiere, Jinwoo Shin, and Yung Yi.
\newblock Multi-armed bandit with additional observations.
\newblock {\em ACM SIGMETRICS Performance Evaluation Review}, 46:53--55, 06
  2018.

\end{thebibliography}
\newcommand{\etalchar}[1]{$^{#1}$}

\newpage

\appendix 


\section{ML Estimator and Concentration Bounds: Omitted Proofs}

\subsection{Proof of \texorpdfstring{\cref{lemma:martingale}}{} }
\label{appendix:martingale}

\lemmaSupermartingale*

\begin{proof}

\textbf{(a)} We first show that the sequence $(G_{t'})_{t' \in \mathbb{N}_{\geq 0}}$ is a super-martingale. For the base case where $t' = t+1$, we have $\Ex{}{\tilde{G}_{t + 1} \mid \mathcal{F}_{t}} = 0 \leq \tilde{G}_t$. Let us now fix any time $t' \leq t$. Since that $\sigma_{t'}$ is $\mathcal{F}_{t'-1}$-measurable and $Z_{t'}$ is sub-Gaussian with variance proxy $\sigma^2_{t'}$, for any $\lambda' \in \mathbb{R}$ we have that $\Ex{}{\exp\left(\lambda' Z_{t'} - \frac{\lambda'^2 \sigma^2_{t'}}{2} \right) \big| \mathcal{F}_{t'-1}} \leq 1$. Thus, by setting $\lambda' = \frac{\lambda}{\sigma^2_{t'}}$, we get that $\Ex{}{\exp\left(\lambda \frac{Z_{t'}}{\sigma^2_{t'}} - \frac{\lambda^2}{2 \sigma^2_{t'}} \right) \big| \mathcal{F}_{t'-1}} \leq 1$. Using that, we have
\begin{align*}
   \Ex{}{\tilde{G}_{t'} \mid \mathcal{F}_{t'-1}} &= \Ex{}{\exp\left(\lambda W_{t'} - \frac{\lambda^2 \nt_{t'}}{2} \right) \big| \mathcal{F}_{t'-1}} \\
   &= \Ex{}{\exp\left(\sum^{t'}_{\tau=1}\left(\lambda \frac{Z_{\tau}}{\sigma^2_{\tau}} - \frac{\lambda^2}{2 \sigma^2_{\tau}} \right) \right) \big| \mathcal{F}_{t'-1}} \\
   &= \tilde{G}_{t'-1} \cdot \Ex{}{\exp\left(\lambda \frac{Z_{t'}}{\sigma^2_{t'}} - \frac{\lambda^2}{2 \sigma^2_{t'}} \right) \big| \mathcal{F}_{t'-1}} \leq \tilde{G}_{t'-1},
\end{align*}
thus proving that $(G_{t'})_{t'}$ is a super-martingale. In addition, since $\phi$ satisfies $\phi \leq t + 1$ almost surely, by Doob's optional stopping theorem, we can conclude that $\Ex{}{G_{t'}} \leq 1$ for all $t' \in \mathbb{N}_{\geq 0}$.

\textbf{(b)} We focus on bounding the probability that $W_{\phi} > \sqrt{2\alpha \nt_{\phi} \log{t}}$ and $\phi \leq t$, since the other tail bound follows by symmetry. 
By denoting $G_{t'} = \exp\left(\lambda(W_{t'} - \sqrt{2\alpha \nt_{t'} \log{t}}) \right) \cdot \event{t' \leq t}$ for some $\lambda > 0$, we get
\begin{align*}
    \Pr\bigg[W_{\phi} > \sqrt{2\alpha \nt_{\phi} \log{t}} \text{ and }\phi \leq t \bigg] = \Pr\bigg[G_{\phi} \geq 1 \bigg] \leq \Ex{}{G_{\phi}},
\end{align*}
where the last inequality follows by Markov, given that $G_{\phi}$ is by construction a non-negative random variable. Thus, in order to complete the proof, it suffices to upper bound $\Ex{}{G_{\phi}}$. 

By setting $\tilde{G}_{t'} = \exp\left(\lambda W_{t'} - \frac{\lambda^2 \nt_{t'}}{2} \right) \cdot \event{t' \leq t}$, we can rewrite 
\begin{align*}
G_{t'} = \tilde{G}_{t'} \cdot \exp\left(\frac{\lambda^2 \nt_{t'}}{2} - \lambda \sqrt{2 \alpha \nt_{t'} \log t} \right).
\end{align*}
Now, for $t' = \phi$, the event that $\phi \leq t$ implies by definition that $\nt_\phi \in I$ and, thus, $L \leq \nt_{\phi} \leq H$. Therefore, by setting $\lambda = \frac{1}{H} \sqrt{2 \alpha L \log t}$, we get that 
\begin{align*}
G_{\phi} &= \tilde{G}_{\phi} \cdot \exp\left(\frac{\lambda^2 \nt_{\phi}}{2} - \lambda \sqrt{2 \alpha \nt_{\phi} \log t} \right) \\
&\leq \tilde{G}_{\phi} \cdot \exp\left(\frac{\lambda^2 H}{2} - \lambda \sqrt{2 \alpha L \log t} \right) \\
&=  \tilde{G}_{\phi} \cdot \exp\left(- \frac{\alpha \cdot L}{H} \log t \right).
\end{align*}

Finally, by using the fact that $\Ex{}{\tilde{G}_{\phi}} \leq 1$ for any $\lambda > 0$, as proved in part (a) of this Lemma, we have that $\Ex{}{G_{\phi}} \leq \exp\left(- \alpha \frac{L}{H} \log t \right) = t^{- \alpha \frac{L}{H}}$, which concludes the proof. 

\textbf{(c)} As in the proof of part (b), we can restrict ourselves in bounding the probability that $W_{\psi} > \nt_\psi\epsilon$ and $\psi \leq t$, while the other tail bound follows by symmetry. 
By denoting $G'_{t'} = \exp\left(\lambda(W_{t'} - \nt_{t'} \epsilon) \right) \cdot \event{t' \leq t}$ for some $\lambda > 0$, we get
\begin{align*}
    \Pr\bigg[W_{\psi} > \nt_{\psi} \text{ and }\psi \leq t \bigg] = \Pr\bigg[G'_{\psi} \geq 1 \bigg] \leq \Ex{}{G'_{\psi}},
\end{align*}
where the last inequality follows by Markov. 

In order to to upper bound $\Ex{}{G'_{\phi}}$, by setting $\tilde{G}_{t'} = \exp\left(\lambda W_{t'} - \frac{\lambda^2 \nt_{t'}}{2} \right) \cdot \event{t' \leq t}$, we can rewrite 
\begin{align*}
G'_{t'} = \tilde{G}_{t'} \cdot \exp\left(\frac{\lambda^2 \nt_{t'}}{2} - \lambda \epsilon \nt_{t'}\right).
\end{align*}
By setting $\lambda = 4 \epsilon$, we can upper bound $G'_{t'}$ as 
$G'_{t'} \leq \tilde{G}_{t'} \cdot \exp\left(- \frac{\nt_{t'} \cdot \epsilon^2}{2}\right).$. Further, if $\psi \leq t$, then by definition we have that $\nt_{\psi} \geq r$, which implies that
$$
G'_{\psi} \leq \tilde{G}_{\psi} \cdot \exp\left(- \frac{ r \cdot \epsilon^2}{2}\right).
$$
The proof follows by taking expectation in the above expression and using the fact that $\Ex{}{\tilde{G}_{\psi}} \leq 1$ for any $\lambda > 0$, as we show in part (a).

\end{proof}

\subsection{Proof of \texorpdfstring{\cref{lemma:anytime}}{}}

\lemmaAnytime*

\begin{proof}
Recall that $\nt_i(t)$ is defined as $\nt_i(t) = \sum_{j \in [K]} N_j(t) / \sigma^2_{j,i}$ and that $\sigma^{\min}_i = \min_{j \in [K]} \sigma_{j,i}$. Without loss of generality, we can assume that $\sigma^{\min}_i > 0$, since otherwise the lemma follows trivially. By time $t$, $\nt_i(t)$ can only take values in the range $\left[\frac{1}{(\sigma^{\min}_i)^2}, \frac{t}{(\sigma^{\min}_i)^2}\right]$. By partitioning the range of $\nt_i(t)$ into intervals of the form $\left[\frac{s}{(\sigma^{\min}_i)^2}, \frac{s+1}{(\sigma^{\min}_i)^2}\right]$ for every $s \in [t-1]$, and applying union bound, we get
\begin{align*}
    &\Pr\left[|\mh_i(t) - \mu_i| > \sqrt{\frac{2a\log{t}}{\nt_i(t)}} \right] \\&= \Pr\left[ \bigcup_{s \in [t-1]}\left( |\mh_i(t) - \mu_i| > \sqrt{\frac{2a\log{t}}{\nt_i(t)}} \text{ and } \nt_i(t) \in \left[\frac{s}{(\sigma^{\min}_i)^2}, \frac{s+1}{(\sigma^{\min}_i)^2}\right] \right) \right] \\
    &\leq \sum_{s \in [t-1]} \Pr\left[|\mh_i(t) - \mu_i| > \sqrt{\frac{2a\log{t}}{\nt_i(t)}} \text{ and } \nt_i(t) \in \left[\frac{s}{(\sigma^{\min}_i)^2}, \frac{s+1}{(\sigma^{\min}_i)^2}\right]\right].
\end{align*}

Notice that $\left|\mh_i(t) - \mu_i\right| > \sqrt{\frac{2a\log{t}}{\nt_i(t)}}$ is equivalent to $\left|\sum_{\tau=1}^{t-1} \frac{X_{i,\tau} - \mu_i}{\s_{i_\tau,i}^2} \right| > \sqrt{2a \nt_i(t)\log{t}}$. By setting $Z_{t'} = X_{i,t'} - \mu_i$, we can observe that the sequence $\{Z_{t'}\}_{t' \leq t}$ satisfies the conditions of \cref{lemma:martingale}. For any fixed $s \in [t-1]$, let us define $I_i(s) = \left[\frac{s}{(\sigma^{\min}_i)^2}, \frac{s+1}{(\sigma^{\min}_i)^2}\right]$. Hence, for $L = \frac{s}{(\sigma^{\min}_i)^2}$ and $H = \frac{s+1}{(\sigma^{\min}_i)^2}$, then if $\left|\sum_{\tau=1}^{t-1} \frac{X_{i,\tau} - \mu_i}{\s_{i_\tau,i}^2} \right| > \sqrt{2a \nt_i(t)\log{t}} \text{ and } \nt_i(t) \in I_i(s)$, then there must exist a stopping-time $\phi$, as described in part $(b)$ of \cref{lemma:martingale} for $I = I_i(s)$, such that $\left| W_{\phi} \right| > \sqrt{2\alpha \nt_{\phi} \log{\phi}} \text{ and }\phi \leq t$. Thus, we get

\begin{align*}
    &\Pr\left[\left|\mh_i(t) - \mu_i\right| > \sqrt{\frac{2a\log{t}}{\nt_i(t)}} \text{ and } \nt_i(t) \in I_i(s)\right] \\
    &= \Pr\left[\left|\sum_{\tau=1}^{t-1} \frac{X_{i,\tau} - \mu_i}{\s_{i_\tau,i}^2} \right| > \sqrt{2a \nt_i(t)\log{t}} \text{ and } \nt_i(t) \in I_i(s)\right]. \\
    & \leq \Pr\left[\left| W_{\phi} \right| > \sqrt{2\alpha \nt_{\phi} \log{\phi}} \text{ and }\phi \leq t \right] \leq  2 \cdot t^{-\alpha \frac{s}{s+1}}.
\end{align*}

Now, given that $\frac{s}{s+1}$ is bounded uniformly from below by $1/2$ for any $s \geq 1$, we can conclude that:
\begin{align*}
    \Pr\left[|\mh_i(t) - \mu_i| > \sqrt{\frac{2a\log{t}}{\nt_i(t)}} \right] &< 2 \cdot t^{1-\alpha/2}.
\end{align*}
\end{proof}


\section{Asymptotic Regret Lower Bound: Omitted Proofs}\label{appendix:lower_bound}

\subsection{Proof of \texorpdfstring{\cref{prop:divergence}}{}}

\propDivergence*
\begin{proof}

We denote by $i_{\tau}$ the arm selected at round $\tau$ by policy $\pi$. Then, for the KL-divergence between the distributions $P,P'$ we have that
\begin{align*}
    D(P||P') 
    &= \sum_{\tau=1}^t\Ex{\nu}{D(P_{i_{\tau}}||P_{i_{\tau}}')} = \sum_{{\tau}=1}^t \sum_{i=1}^K \Ex{\nu}{D(P_{i}||P_{i}') \event{i_{\tau}=i}}
    = \sum_{i=1}^K D(P_{i}||P_{i}')  \Ex{\nu}{N_i(t)}.
\end{align*}
Note that, here $P_k,P_k'$ are multivariate distributions of $K$ arms. We denote by $P_i^{(j)}$ the distribution that corresponds to the $j$-th coordinate of $P_i$. Since the distributions $\{P_i^{(j)}\}_{j\in[K]}$ are independent, we have that
$    D(P_{i}||P_{i}') = \sum_{j=1}^K D(P_i^{(j)},P_i^{(j)'}) $.
Therefore
\begin{align*}
    D(P||P') = \sum_{i=1}^K \sum_{j=1}^K D(P_i^{(j)},P_i^{(j)'}) \Ex{\nu}{N_i(t)} = \sum_{i=1}^K \Ex{\nu}{N_i(t)} \frac{(\mu_k-\mu_k')^2}{2\s_{i,k}^2},
\end{align*}
where in the last equation we used the fact that for any $i\in[K]$, the distributions $P_i^{(k)},P_i^{(k)'}$ are Gaussian with different means and same variance.  

\end{proof}

\subsection{Proof of \texorpdfstring{\cref{thm:lower_bound}}{}}

\thmLowerBound*
\begin{proof}
Let us fix any consistent policy $\pi$. Let $\mu^{(0)}$ be a mean-reward vector and $k$ a suboptimal arm in $\mu^{(0)}$. We define vector $\mu^{(1)}$ such that
\begin{align}\label{eq:huber}
    \mu_i^{(1)}=
    \begin{cases}
    \mu_i^{(0)}, \text{ if }i\not = k \\
    \mu^{(0)*}+\epsilon, \text{ if }i = k
    \end{cases}.
\end{align}
for some $\epsilon>0$. For any event $\mathcal{Q}$, due to the Bretagnolle-Huber inequality \cite{tsybakov2008} we have that
\begin{align*}
    \Pr_{(0)}[\mathcal{Q}] + \Pr_{(1)}[\mathcal{Q}^c] \geq \frac{1}{2} e^{-D\left(\mu^{(0)}||\mu^{(1)}\right)},
\end{align*}
where $\Pr_{(0)}$ (resp. $\Pr_{(1)}$) denotes the probability over the randomness induced by the interplay of the reward distribution with the mean vector $\mu_{(0)}$ (resp. $\mu_{(1)}$) and the policy. 

We define the event $\mathcal{Q}=\left\{N_k(t)>\frac{t}{2}\right\}$. Since $k$ is a suboptimal arm in $\mu_{(0)}$ and optimal in $\mu_{(1)}$, for the regret of policy $\pi$ in the two environments we have that
\begin{align*}
    R^{\pi}_t(\mu^{(0)}) + R^{\pi}_t(\mu^{(1)})
    &\geq \frac{t}{2} \Pr_{(0)}[\mathcal{Q}]\D_k(\mu^{(0)}) + \frac{t}{2}\Pr_{(1)}[\mathcal{Q}^c] \epsilon
\end{align*}
Therefore, we have that
\begin{align*}
    R^{\pi}_t(\mu^{(0)}) + R^{\pi}_t(\mu^{(1)})
    &\geq \frac{t\cdot \min\{ \D_k(\mu^{(0)}),\epsilon\}}{2} \left(\Pr_{(0)}[\mathcal{Q}] + \Pr_{(1)}[\mathcal{Q}^c] \right) \\
    &\geq \frac{t\cdot \min\{ \D_k(\mu^{(0)}),\epsilon\}}{4} e^{-D(\mu^{(0)}||\mu^{(1)})}\\
    &\geq \frac{t\cdot \min\{ \D_k(\mu^{(0)}),\epsilon\}}{4} e^{-\sum_{j:k\in S_j} \Ex{(0)}{N_j(t)} \frac{(\D_k(\mu^{(0)})+\epsilon)^2}{2\s_{j,k}^2}}.
\end{align*}
where in the second to last inequality we used \cref{eq:huber}, while for the last inequality we used \cref{prop:divergence}. By rearranging terms we obtain 
\begin{align}\label{eq:lb1}
    \sum_{j:k\in S_j} \Ex{(0)}{N_j(t)} \frac{(\D_k(\mu^{(0)})+\epsilon)^2}{2\s_{j,k}^2} 
    \geq \log{\frac{t \cdot \min\{ \D_k(\mu^{(0)}),\epsilon\}}{4 (R^{\pi}_t(\mu^{(0)}) + R^{\pi}_t(\mu^{(1)}))}}.
\end{align}

We recall that for any consistent policy $\pi$, for any environment $\mu$ and $p>0$ we have that
\begin{align*}
    \lim_{t\rightarrow \infty} \frac{R_t^{\pi}(\mu)}{n^p} = 0.
\end{align*}

Thus, by rearranging terms in \cref{eq:lb1}, dividing by $\log{t}$ and taking the limit below, we have that
\begin{align}\label{eq:lb2}
&\liminf_{t\rightarrow \infty} \frac{\sum_{i\in[K]} \frac{\Ex{(0)}{N_i(t)}}{\s_{i,k}^2}}{\log{t}} \nonumber\\
&\geq \liminf_{t\rightarrow \infty} \frac{2}{(\D_k(\mu^{(0)})+\epsilon)^2} \frac{\log{\frac{t \cdot \min\{ \D_k(\mu^{(0)}),\epsilon\}}{4 (R^{\pi}_t(\mu^{(0)}) + R^{\pi}_t(\mu^{(1)}))}}}{\log{t}} \nonumber \\
&= \frac{2}{(\D_k(\mu^{(0)})+\epsilon)^2} \liminf_{t\rightarrow \infty}  \frac{
\log{t} 
- \log{ (R^{\pi}_t(\mu^{(0)}) + R^{\pi}_t(\mu^{(1)}))} 
+ \log{\frac{\min\{ \D_k(\mu^{(0)}),\epsilon\}}{4}}}{\log{t}}.
\end{align}
We have that
\begin{align*}
&\liminf_{t\rightarrow \infty}  \frac{
\log{t} 
- \log{ (R^{\pi}_t(\mu^{(0)}) + R^{\pi}_t(\mu^{(1)}))} 
+ \log{\frac{\min\{ \D_k(\mu^{(0)}),\epsilon\}}{4}}}{\log{t}} \\
&\quad \quad= \left(1 -  \liminf_{t\rightarrow \infty}  \frac{\log{ (R^{\pi}_t(\mu^{(0)}) + R^{\pi}_t(\mu^{(1)}))} }{\log{t}}
\right) \\
&\quad \quad=1~.
\end{align*}
Therefore the bound in \cref{eq:lb2} becomes
\begin{align*}
\liminf_{t\rightarrow \infty} \frac{\sum_{i\in[K]} \frac{\Ex{(0)}{N_i(t)}}{\s_{i,k}^2}}{\log{t}}
\geq  \frac{2}{(\D_k(\mu^{(0)})+\epsilon)^2}~.
\end{align*}

By following a similar procedure for all arms, the theorem follows by taking $\epsilon\rightarrow 0$ in the above bound and due to the construction of the objective and constraints of the LP: since the quantity $\frac{\Ex{(0)}{N_i(t)}}{\log{t}}$ needs to satisfy the constraints of the LP asymptotically, then the minimum-regret such solution (which is computed by the LP objective) provides an asymptotic lower bound on the regret.  

\end{proof}


\section{Algorithm and Regret Analysis: Omitted Proofs}
\label{appendix:upper_bound}

\subsection{Proof of \texorpdfstring{\cref{thm:regret_upper_bound}}{}}
Here we prove the main result of \cref{section:Algorithm}:
\thmRegretUpperBound*
\begin{proof}
We define the following events, which correspond to the conditions that are examined in Cases (1) and (2) of \cref{algorithm1}: 
\begin{align*}
    &\Vt = \left\{\left(\frac{N_1(t)}{4a\log{t}},...,\frac{N_K(t)}{4a\log{t}}\right)\in C(\mh_t)\right\},\\
    &\Wt = \left\{\min_{i\in [K]} \nt_i(t) < \frac{1}{K}\beta(n_e(t)) \right\}.
\end{align*}
Moreover, we consider the following error events for the ML estimator $\mh_t$:
\begin{align*}
    &\Ut = \left\{|\mh_i(t) - \mu_i|\leq \sqrt{\frac{2a\log{t}}{\nt_i(t)}}, ~~\forall i\in [K]\right\},\\
    &\Yt = \left\{|\mh_i(t) - \mu_i|\leq \epsilon, ~~\forall i\in [K]\right\}.
\end{align*}
Using the above definitions, we can provide an upper bound on the regret accumulated by \cref{algorithm1} within $T$ rounds as a decomposition of the above events, as follows: 
\begin{align}\label{eq:regret1}
    R_T(\mu) &= T \mu^*  -  \Ex{}{\sum_{t\in[T]} X_{i_t,t}}
    =\Ex{}{\sum_{t\in [T]} \D_{t}(\mu)} \nonumber\\
    &\leq K\D_{\max} + \sum_{t=K+1}^T \Ex{}{\D_{t}(\mu) \left(\event{\Ut^c} + \event{\Ut,\Vt} + \event{\Ut,\Vt^c,\Wt}+ \event{\Ut,\Vt^c,\Wt^c,\Yt}+ \event{\Ut,\Vt^c,\Wt^c,\Yt^c} \right) }.
\end{align}
In the rest of this proof, we focus on each one of the above terms separately.

\medskip
\paragraph{Regret due to $\event{\Ut^c}$.} The regret accumulated due to the event $\event{\Ut^c}$ over $T$ rounds can be bounded by using union bound over all arms and then applying the concentration result of the Maximum-Likelihood estimator proved in \cref{lemma:anytime}. In particular, we obtain that:
\begin{align}\label{eq:regret2}
    \sum_{t=K+1}^{T}\event{\Ut^c} \nonumber
    &\leq \sum_{t=K+1}^T \sum_{i\in[K]} \Pr\left[|\mh_i(t) - \mu_i|> \sqrt{\frac{2a\log{t}}{\nt_i(t)}} \right]\nonumber\\
    &\leq K\sum_{t=K+1}^T 2 t^{1-\alpha/2} \nonumber\\
    &\leq \frac{4K}{\alpha-4}
\end{align}
Therefore, in what follows we assume the the event $\Ut$ holds for every $t\in [T]$.  

\medskip
\paragraph{Regret due to $\event{\Ut,\Vt}$.} As we argued in the description of the algorithm in \cref{section:Algorithm}, when these events hold we show that the greedy arm selection of \cref{algorithm1} in Case (1) leads to the selection of the optimal arm. 
By using the definitions of the events $\Ut$ and $\Vt$ we have that the error in the estimate of the mean of any arm $i\in [K]$ that is suboptimal in the vector of estimations $\mh(t)$ can be upper bounded as:
\begin{align*}
    |\mh_i(t) - \mu_i| \leq \sqrt{\frac{2a\log{t}}{\nt_i(t)}} \leq \sqrt{\frac{2a\log{t}}{\frac{8\alpha\log{t}}{\D_i^2(\mh(t))}}} = \frac{\D_i(\mh(t))}{2}, ~~ \forall i\not = i^*(\mh(t)).
\end{align*}
In particular, this implies that for any suboptimal arm $i$ of $\mh(t)$, we have that
\begin{align}\label{eq:ineq1}
    \mu_i - \mh_i(t) \leq \frac{\D_i(\mh(t))}{2},
\end{align}
while, similarly, for the optimal arm of $\mh(t)$, $i^*(\mh(t))$, we have
\begin{align}\label{eq:ineq2}
    \mh_{i^*(\mh(t))}(t) - \mu_{i^*(\mh(t))} \leq \frac{\D_{\min}(\mh(t))}{2}.
\end{align}
By adding \cref{eq:ineq1} and \cref{eq:ineq2} we get that:
\begin{align*}
    \mu_i - \mh_i(t) + \mh_{i^*(\mh(t))}(t) - \mu_{i^*(\mh(t))} \leq \frac{\D_{i}(\mh(t))}{2} + \frac{\D_{\min}(\mh(t))}{2} \leq \D_{i}(\mh(t)),
\end{align*}
which, by eliminating $\D_{i}(\mh(t))$ and $- \mh_i(t) + \mh_{i^*(\mh(t))}(t) - \mu_{i^*(\mh(t))}$ from both sides, gives that the regret of any arm $i$ that is suboptimal in $\mh(t)$ satisfies: $\mu_i\leq \mu_{i^*(\mh(t))}$. Therefore the optimal arm of the vector of estimates at time $t$, $i^*(\mh(t))$, corresponds to the optimal arm of vector $\mu$ and
\begin{align}\label{eq:regret3}
    \D_t(\mu)\event{\Ut,\Vt}=0
\end{align}
thus, the rounds where the events $\Ut,\Vt$ hold do not contribute to the regret. 

\medskip
\paragraph{Regret due to $\event{\Ut,\Vt^c,\Wt}$.}
We focus on the regret accumulated during the rounds where $\Ut,\Vt^c,\Wt$ holds. These correspond to the the rounds where  \cref{algorithm1} is attempting to ensure uniform exploration for all arms in terms of their weighted numbers of samples. 
The term of the regret depending on $\event{\Ut,\Vt^c,\Wt}$ can be bounded as follows: First we develop the following counting argument for the number of times the playing satisfies the events $\Vt^c,\Wt$ compared to the total number of ``exploration rounds'', i.e. the number of times the algorithm enters Case (2) or Case (3):

\lemmaCounting*
\begin{proof}
Before we prove this result we show the following proposition. This proposition states that if an algorithm satisfies  $\{\Vt^c,\Wt\}$ (which corresponds to Case (2) of \cref{algorithm1}) at least $K$ times, then the minimum weighted number of samples, $\nt_i(t)$, for any arm $i\in [K]$ after these rounds, will have increased by at least $\frac{1}{\sm^2}$. 
\begin{proposition}\label{prop:pigeonhole}
Let $t_2>t_1>K$, then if  $\sum_{t=t_1}^{t_2-1} \event{\Vt^c,\Wt}\geq K$ then $$\min_{i\in [K]}\nt_i(t_2) \geq \min_{i\in [K]}\nt_i(t_1) + \frac{1}{\sm^2}.$$
\end{proposition}
\begin{proof}
    This proposition follows from a simple pigeonhole argument: for at least $K$ consecutive rounds an arm with minimum weighted number of samples collected $i_t=\min_{i\in [K]}\nt_i(t)$ is being identified and subsequently the algorithm chooses an arm $j$ such that $\frac{1}{\s_{j,i_t}^2}$ is the maximum possible over $j\in [K]$ for the arm $i_t$. Suppose that there exists an arm $j\in[K]$ with minimum weighted number of samples at time $t_1$, i.e. $\nt_{j}(t_1) = \min_{i\in[K]}\nt_{i}(t_1)$ that was not explored during these (at least $K$) rounds. Then, this means that other arms with minimum weighted number of samples were explored during these rounds. Since the arms are $K$, and the exploration step in Case (2) of the \cref{algorithm1} is defined such that the weighted number of samples strictly increases for the arm that is currently being explored, the assumption that $j$ was not explored during these rounds leads to a contradiction. 
    Moreover, by definition of $\sm$ the weighted number of samples for any such arm $i_t$ will have increased by at least $\frac{1}{\s_{j,i}^2}\geq \frac{1}{\sm^2}$. Thus, at the end of at least $K$ such rounds, the minimum $\nt_i(t)$ over $i\in [K]$ is guaranteed to increase by at least $\frac{1}{\sm^2}$.
\end{proof}

Now let $t'$ be the last round where the event $\{\Vt^c,\Wt\}$ holds, i.e.:
\begin{align*}
    t' = \max\{t : \event{\Vt^c,\Wt}, K \leq t\leq T\}.
\end{align*}
Due to \cref{prop:pigeonhole}, we have that at time $t'$ the minimum value of the weighted number of samples of any arm $i\in[K]$ satisfies:
\begin{align*}
    \min_{i\in [K]}\nt_i(t') \geq \frac{1}{K}\sum_{t=K+1}^{t'} \frac{\event{\Vt^c,\Wt}}{\sm^2} .
\end{align*}
By rearranging terms in the above and using that when $\{\Vt^c,\Wt\}$ holds we have that $\min_{i\in [K]}\nt_i(t') < \frac{1}{K}\beta\left({n_e(t')}\right)$ we obtain the following bound on the number of times that the event $\{\Vt^c,\Wt\}$ occurs over $T$ rounds:
\begin{align}
    \sum_{t=K+1}^{T} \event{\Vt^c,\Wt}  
    &\leq \sum_{t=K+1}^{t'} \event{\Vt^c,\Wt} +1 \nonumber\\
    &\leq K\sm^2\min_{i\in [K]}\nt_i(t') +1 \nonumber\\
    &< \sm^2\beta\left({n_e(t')}\right) +1. \label{eq:betas}
\end{align}

Now we use that $n_e(t')=\sum_{t=K+1}^{t'} \event{\Vt^c}$ and $t'\leq T$: 
\begin{align*}
    \cref{eq:betas} = \sm^2\beta\left(\sum_{t=K+1}^{t'} \event{\Vt^c}\right) +1 \leq \sm^2\beta\left(\sum_{t=K+1}^{T} \event{\Vt^c}\right) +1.
\end{align*}

Then using the definition of $\beta(x)=\frac{x^{\gamma}}{2}$ we obtain that: 

\begin{align*}
    \cref{eq:betas}\leq \frac{\sm^2}{2} \frac{\left(\sum_{t=K+1}^T \event{\Vt^c}\right)^{\gamma}}{\sm^{2}} +1 = \frac{\left(\sum_{t=K+1}^T \event{\Vt^c}\right)^{\gamma}}{2} + 1.
\end{align*}

\end{proof}

Equipped with \cref{lemma:counting} and using a decomposition of the event $\Vt^c$ we can bound the term of the regret depending on $\event{\Ut,\Vt^c,\Wt}$ as follows:
\begin{align*}
    &\sum_{t=K+1}^T \event{\Ut,\Vt^c,\Wt} \\
    &\leq \sum_{t=K+1}^T \event{\Vt^c,\Wt}\\
    &\leq \frac{\left(\sum_{t=K+1}^T \event{\Vt^c}\right)^{\gamma}}{2} + 1\\
    &\leq \frac{\left(\sum_{t=K+1}^T \event{\Ut^c} + \event{\Ut,\Vt^c,\Wt}+ \event{\Ut,\Vt^c,\Wt^c,\Yt}+ \event{\Ut,\Vt^c,\Wt^c,\Yt^c} \right)^{\gamma}}{2} + 1\\
    &\leq \frac{1}{2}\sum_{t=K+1}^T \left(\event{\Ut^c}
    + \event{\Ut,\Vt^c,\Wt}
    + \event{\Ut,\Vt^c,\Wt^c,\Yt^c}\right)
    + \frac{\left(\sum_{t=K+1}^T \event{\Ut,\Vt^c,\Wt^c,\Yt}\right)^{\gamma}}{2}
    +1,
\end{align*}
where the last inequality comes from the assumption that $\beta(.)$ is a sub-additive function when when its arguments are larger than $1$ (a fact that we can assume that trivially holds, otherwise this implies that the regret in these rounds is $0$). By rearranging and grouping terms we obtain the following bound:
\begin{align}\label{eq:regret4}
    \sum_{t=K+1}^T \event{\Ut,\Vt^c,\Wt} \leq \sum_{t=K+1}^T \left(\event{\Ut^c}
    + \event{\Ut,\Vt^c,\Wt^c,\Yt^c}\right)
    + \left(\sum_{t=K+1}^T \event{\Ut,\Vt^c,\Wt^c,\Yt}\right)^{\gamma}
    + 2.
\end{align}
Thus, in order to bound the regret in the case where $\Ut,\Vt^c,\Wt^c$ hold, we need to bound the regret accumulated due to the terms $\event{\Ut,\Vt^c,\Wt^c,\Yt^c}$ and $\event{\Ut,\Vt^c,\Wt^c,\Yt}$. 

\medskip
\paragraph{Regret due to $\event{\Ut,\Vt^c,\Wt^c,\Yt}$.}
When the event $\Vt^c,\Wt^c$ holds the constraints of the currently estimated LP are not satisfied. In this case, the algorithm selects an arm $j$ that violates these constraints, i.e. $\frac{N_j(t)}{4a\log{t}} < c^*_j(\mh_t)$. In addition, we recall that the event $\Yt$ implies that the error in all current mean estimates is less that $\epsilon$, i.e. $|\mh_{j,t} - \mu_j|\leq \epsilon$. For the case where all estimates are $\epsilon$-close to the original values, we can construct an $\epsilon$-approximate LP. Therefore, in this case we are able to use the worst-case $\epsilon$-approximate LP solution introduced in \cref{def:epsilonLP}:
\begin{align*}
    c^*_j(\mu,\epsilon) = \sup_{|\mu'_i-\mu_i|\leq \epsilon, \forall i} c^*_j(\mu').
\end{align*}

\lemmaDominantTerm* 
\begin{proof}
Let $t'$ be the last round such that $t'\leq T$ where $\{\Ut,\Vt^c,\Wt^c,\Yt\}$ holds and the algorithm selects arm $j$. Then we have that:
\begin{align*}
    \sum_{t=K+1}^T \event{\Ut,\Vt^c,\Wt^c,\Yt, i_t=j} &\leq \sum_{t=K+1}^{t'} \event{\Ut,\Vt^c,\Wt^c,\Yt, i_t=j} \\
    &\leq N_j(t') \\
    &\leq c^*_j(\mh_{t'}) 4\alpha \log{t'} \\
    &\leq c^*_j(\mu,\epsilon) 4\alpha \log{T}, 
\end{align*}
where the third to last inequality comes from the fact that $\frac{N_j(t)}{4a\log{t}} < c^*_j(\mh_t)$. The second to last is due to $t'\leq T$ and the last inequality follows from \cref{def:epsilonLP} and the fact that $\mh_{i,t'}$ is $\epsilon-$close to $\mu_i$ for every $i\in [K]$. 
\end{proof}

Using the same idea for all arms, we bound the term of the regret depending on $\event{\Ut,\Vt^c,\Wt^c,\Yt}$ as follows:
\begin{align}\label{eq:regret5}
    \sum_{t=K+1}^T \D_t \event{\Ut,\Vt^c,\Wt^c,\Yt} 
    &\leq \sum_{j\in[K]} \D_j(\mu) \event{\Ut,\Vt^c,\Wt^c,\Yt, i_t=j} \nonumber\\
    &\leq 4\alpha \sum_{j\in[K]} \D_j(\mu) c^*_j(\mu,\epsilon)  \log{T} .
\end{align}

\medskip
\paragraph{Regret due to $\event{\Ut,\Vt^c,\Wt^c,\Yt^c}$.}

We can bound the term of the regret due to the event $\event{\Ut,\Vt^c,\Wt^c,\Yt^c}$ as follows:
\begin{align}
    &\Ex{}{\sum_{t=K+1}^T \D_t \event{\Ut,\Vt^c,\Wt^c,\Yt^c}}\nonumber\\ 
    &\leq \D_{\max} \Ex{}{\sum_{t=K+1}^T \event{\Ut,\Vt^c,\Wt^c,\Yt^c}}\nonumber\\
    &= \D_{\max} \Ex{}{\sum_{t=K+1}^T \event{~\Ut,~\Vt^c, ~\min_{j\in [K]} \nt_j(t) \geq \frac{1}{K}\beta\left(n_e(t)\right),  ~\exists i\in [K]: |\mh_i(t) - \mu_i|> \epsilon~}}\nonumber\\
    &\leq \D_{\max} \Ex{}{\sum_{i\in[K]}\sum_{t=K+1}^T \event{~\Ut,~\Vt^c, ~\min_{j\in [K]} \nt_j(t) \geq \frac{1}{K}\beta\left(n_e(t)\right),  |\mh_i(t) - \mu_i|> \epsilon~}} \label{eq:sum_stopings} .
\end{align}
where in the last inequality above we union bounded on the events of error in the estimates of any arm. 


Finally, for better manipulation of the bound in \cref{eq:sum_stopings}, we construct the following sets :
\begin{align*}
    \Lambda=\{t\in [T] : \Ut,\Vt^c, \min_{i\in[K]}\nt_i(t) \geq \frac{1}{K}\beta\left(n_e(t)\right)\}
\end{align*}
and the more refined
\begin{align*}
    \Lambda(\tau)= \{t\in[T]:\Ut,\Vt^c, \min_{i\in[K]}\nt_i(t) \geq \frac{1}{K}\beta\left(\tau\right), n_e(t)=\tau\}
\end{align*}
Note that if $t\in \Lambda(\tau)$ then $\nt_i(t) \geq \frac{1}{K}\beta\left(n_e(t)\right)$ and that $|\Lambda(\tau)|\leq 1$. In addition, $\Lambda \subseteq
\cup_{\tau\in [T]}\Lambda(\tau)$. 
Then, let $\phi_{\tau}$ be a stopping time such that $\phi_{\tau} = t$ if $\Lambda(\tau) = \{t\}$ and $\phi_{\tau} = T+1$ otherwise. We can write that:
    \begin{align*}
        \Ex{}{\sum_{t\in[T]} \event{t\in \Lambda, |\mh_i(t) - \mu_i|> \epsilon}} 
        &\leq \Ex{}{\sum_{\tau\in[T]} \event{\phi_{\tau}\leq T, |\mh_i(t) - \mu_i|> \epsilon}} \\
        &= \sum_{\tau\in[T]} \Pr\left[\phi_{\tau}\leq T, |\mh_i(t) - \mu_i|> \epsilon \right] \\
        &\leq \sum_{\tau\in[T]} 2 \exp\left(-\frac{\epsilon^2}{2} \frac{1}{K}\beta(\tau)\right),
    \end{align*}
where in the last inequality we apply the result of \cref{lemma:martingale}.

Therefore, \cref{eq:sum_stopings} can be bounded as follows:

\begin{align}\label{eq:regret6}
    \cref{eq:sum_stopings} 
    &= \D_{\max} \Ex{}{\sum_{i\in[K]}\sum_{t=K+1}^T \event{t\in\Lambda, |\mh_i(t) - \mu_i|> \epsilon~}} \nonumber\\
    &\leq \D_{\max} \sum_{i\in[K]} \sum_{\tau\in[T]}  \exp\left(-\frac{\beta(\tau) \epsilon^2}{2K}\right)\\
    &= \D_{\max} \sum_{i\in[K]} \sum_{\tau\in[T]}  \exp\left(-\frac{\tau^{\gamma} \epsilon^2}{4K\sm^2}\right).
\end{align}
where in the last inequality we used the definition of $\beta(\cdot)$ function. 

Combining the bounds in \cref{eq:regret1,eq:regret2,eq:regret3,eq:regret4,eq:regret5,eq:regret6} we conclude that the regret of \cref{algorithm1} can be upper bounded by:
\begin{align*}
    R_T(\mu)
    &\leq K\D_{\max} + \D_{\max}\frac{8K}{\alpha-4}+ 
    \D_{\max} \left(4\alpha \sum_{j\in[K]} \D_j(\mu) c^*_j(\mu,\epsilon)  \log{T} \right)^{\gamma}
    \\
    &+ 2\D_{\max} K \sum_{\tau\in[T]}  \exp\left(-\frac{\tau^{\gamma} \epsilon^2}{4K\sm^2}\right) + 4\alpha \sum_{j\in[K]} \D_j(\mu) c^*_j(\mu,\epsilon)  \log{T} + 2.
\end{align*}

\end{proof}

\subsection{Proof of \texorpdfstring{\cref{cor:asymptotic_optimality}}{}}

\corAsymptoticOptimality*

\begin{proof}

Recall that:
\begin{align*}
    R_T(\mu)
    &\leq \left(2K + \frac{8K}{\alpha-4} + 2\right)\D_{\max} + 2\D_{\max} K \sum_{\tau\in[T]}  \exp\left(-\frac{\tau^{\gamma} \epsilon^2}{4K\sm^2}\right)\\
    &+\D_{\max}\left(4\alpha \sum_{j\in[K]} c^*_j(\mu,\epsilon)  \log{T} +K \right)^{\gamma} + 4\alpha \sum_{j\in[K]} \D_j(\mu) c^*_j(\mu,\epsilon)  \log{T}.
\end{align*}
Observe that for any $\epsilon, \sm >0$ and $\gamma\in(0,1)$ we have that $\sum_{\tau\in[T]}  \exp\left(-\frac{\tau^{\gamma} \epsilon^2}{4K\sm^2}\right)<\infty$.
Moreover, the term: $$\left(4\alpha \sum_{j\in[K]} c^*_j(\mu,\epsilon)  \log{T} +K \right)^{\gamma} \in o\left(4\alpha \sum_{j\in[K]} c^*_j(\mu,\epsilon)  \log{T} +K \right)$$ and thus vanishes asymptotically. Therefore, 
\begin{align*}
    \limsup_{T\rightarrow \infty} \frac{R_T(\mu)}{\log{T}} = 4\alpha \sum_{j\in[K]} \D_j(\mu) c^*_j(\mu,\epsilon). 
\end{align*}

\end{proof}

\end{document}